\documentclass[letterpaper, 10 pt, conference]{ieeeconf}  
\IEEEoverridecommandlockouts
\usepackage{cite}
\usepackage{amsfonts}
\usepackage{amssymb}
\usepackage[font=normalsize, skip=2pt]{caption}  
\usepackage[font=small,skip=2pt]{caption}
\usepackage{textcomp}
\usepackage{xcolor}

\usepackage{times}

\usepackage[utf8]{inputenc}
\DeclareUnicodeCharacter{2212}{$-$}
\usepackage{graphicx}
\usepackage[version=4]{mhchem}
\usepackage{siunitx}
\usepackage{algorithm}
\usepackage{algpseudocode}

\usepackage{longtable,tabularx}
\usepackage{nccmath}
\usepackage{setspace}
\usepackage{tabu}
\usepackage{listings}
\usepackage{afterpage}
\usepackage{titlesec}
\usepackage{fancyhdr}
\usepackage{float}
\usepackage[caption=false,font=footnotesize]{subfig}

\usepackage{cleveref}
\usepackage{textcomp}
\usepackage{pict2e}
\usepackage{wasysym}

\usepackage{multicol}

\usepackage{lipsum}
\usepackage{indentfirst}

\def\BibTeX{{\rm B\kern-.05em{\sc i\kern-.025em b}\kern-.08em
    T\kern-.1667em\lower.7ex\hbox{E}\kern-.125emX}}

\usepackage[utf8]{inputenc}
\usepackage{textcomp}
\usepackage{graphicx}
\usepackage[version=4]{mhchem}
\usepackage{siunitx}
\usepackage{longtable,tabularx}
\usepackage{caption}
\usepackage{bm}

\usepackage{amsthm}
\newtheorem{theorem}{Theorem}[section]
\newtheorem{lemma}[theorem]{Lemma}
\newtheorem{proposition}[theorem]{Proposition}
\newtheorem{corollary}[theorem]{Corollary}
\newtheorem{remark}{Remark}

\usepackage{algorithm}
\usepackage{algpseudocode}

\usepackage{comment}
\usepackage{indentfirst}
\usepackage{here}
\usepackage{mathtools}
\usepackage{empheq}
\newcommand{\argmax}{\mathop{\rm arg~max}\limits}
\newcommand{\argmin}{\mathop{\rm arg~min}\limits}

\title{\LARGE \bf NODA-MMH: Certified Learning-Aided Nonlinear Control for Magnetically-Actuated Swarm Experiment Toward On-Orbit Proof}
\author{%
Yuta Takahashi$^{1,2}$, 
Atsuki Ochi$^{1,2}$, 
Yoichi Tomioka$^{3}$, 
Shin-Ichiro Sakai$^{4}$
\thanks{$^{1}$Graduate Student of Mechanical Engineering, Institute of Science Tokyo, Meguro-ku, Tokyo 152-8550, Japan}%
\thanks{$^{2}$Researcher and Engineer of Satellite R\&D Division, Interstellar Technologies Inc., Koto-ku, Tokyo 135‑0016, Japan}%
\thanks{$^{3}$Senior Associate Professor of Computer Science and Engineering, University of Aizu, Aizuwakamatsu, Fukushima 965-0006, Japan}%
\thanks{$^{4}$Professor of Spacecraft Engineering, Institute of Space and Astronautical Science, Dayton, Kanagawa 252-5210, Japan}%
\thanks{Corr. author: Yuta Takahashi, {\tt\small stateofyuta@gmail.com}}%
}
\begin{document}
\maketitle
\thispagestyle{empty}
\pagestyle{empty}
\begin{abstract}
This study experimentally validates the principle of large-scale satellite swarm control through learning-aided magnetic field interactions generated by satellite-mounted magnetorquers. This actuation presents a promising solution for the long-term formation maintenance of multiple satellites and has primarily been demonstrated in ground-based testbeds for two-satellite position control. However, as the number of satellites increases beyond three, fundamental challenges coupled with the high nonlinearity arise: 1) nonholonomic constraints, 2) underactuation, 3) scalability, and 4) computational cost. Previous studies have shown that time-integrated current control theoretically solves these problems, where the average actuator outputs align with the desired command, and a learning-based technique further enhances their performance. Through multiple experiments, we validate critical aspects of learning-aided time-integrated current control: (1) enhanced controllability of the averaged system dynamics, with a theoretically guaranteed error bound, and (2) decentralized current management. We design two-axis coils and a ground-based experimental setup utilizing an air-bearing platform, enabling a mathematical replication of orbital dynamics. Based on the effectiveness of the learned interaction model, we introduce NODA-MMH (Neural power-Optimal Dipole Allocation for certified learned Model-based Magnetically swarm control Harness) for model-based power-optimal swarm control. This study complements our tutorial paper on magnetically actuated swarms for the long-term formation maintenance problem.
\end{abstract}

\section{Introduction}
Large-scale space structures offer numerous advantages to both the scientific and commercial sectors, and distributed architecture will significantly enhance these initiatives. One example is large space antennas for direct-connectibility; they enable high data rates with small ground terminals and provide resilient communication in the cellular band. However, the size of the launch vehicle imposes physical constraints and presents significant challenges for deploying larger antennas. To address this, previous studies \cite{takahashi2025distance,shim2025feasibility} investigated the distributed space antenna system, where multiple satellites form a virtual space antenna shown in Fig. ~\ref{fig:fractionated_spacecraft}. This payload distribution enhances redundancy and mission flexibility, reducing thermal and structural demands. 
\begin{figure}[tb]
  \centering
\hspace*{-0.5cm} 
\begin{minipage}[b]{0.95\columnwidth}
    \centering
    \subfloat[Target generation for experiment evaluation toward orbit proof.]{
\includegraphics[width=\columnwidth]{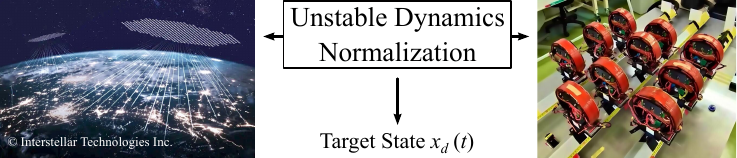}
      \label{fig:distributed_antenna}
    }
  \end{minipage}
 \hspace*{-0.5cm}
  \begin{minipage}[b]{0.95\columnwidth}
    \centering
    \subfloat[Overview of NODA-MMH (Neural power-Optimal Dipole Allocation for certified learned Model-based Magnetically swarm control Harness).]{
      \includegraphics[width=\linewidth]{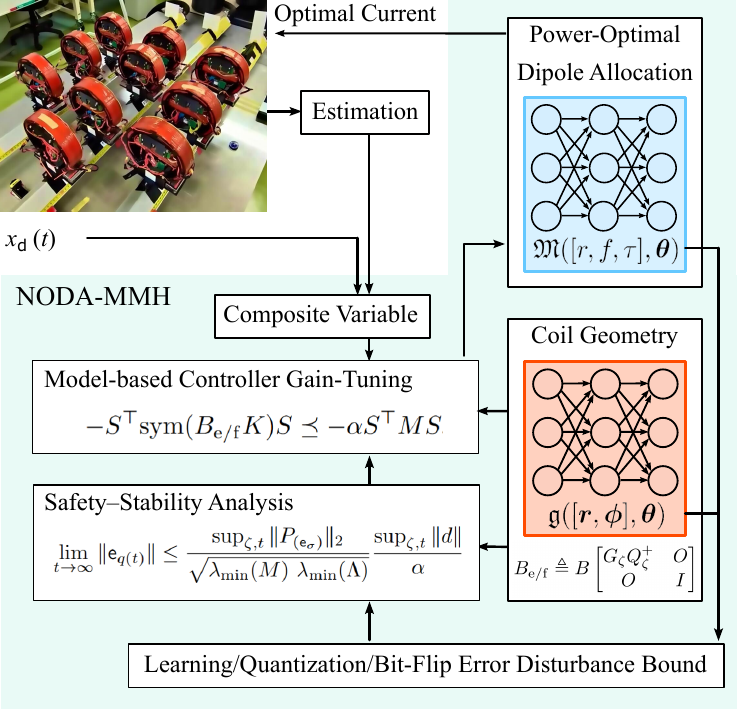}
      \label{fig:uncoordinated}
    }
  \end{minipage}
  \caption{Distributed space antenna concept \cite{takahashi2025distance,shim2025feasibility,takahashi2025anticipatory} and two-axis magnetorquer testbed on a linear air track. 
  The satellite swarm is decentralized to a small group by the time-integrated control, as shall be introduced in subsection~\ref{Time_Integrated_Control_of_Alternating_Current}.
  }
  \label{fig:fractionated_spacecraft}
\end{figure}

Magnetic interaction generated by satellite-mounted magnetorquers (MTQs) presents a promising solution for the long-term maintenance of satellite formation. MTQs are widely employed for attitude control in Earth-orbiting satellites, and Electromagnetic Formation Flight (EMFF) extends this into the relative position control of multiple satellites. The microgravity demonstrations have validated this approach, mainly focusing on relative distance control of two satellites \cite{porter2014demonstration,youngquist2013alternating}. Previous studies have shown the feasibility of a distributed space antenna using MTQs under unstable orbital dynamics of a low-Earth orbit \cite{shim2025feasibility,takahashi2025distance}. 

Despite these advancements, scaling the system to more than three satellites introduces four fundamental challenges. The first one arises from the conservation of angular momentum, which is the nonholonomic constraint. Then, solely smooth state feedback control cannot converge all absolute positions and attitudes of multiple satellites to the desired state \cite{brockett1983asymptotic}. The second problem is underactuation, which stems from a fundamental limitation in controlling the degrees of freedom. Even when each satellite is equipped with a three-axis magnetic coil and applies a direct current (DC), the total number of control inputs is limited to $3N$ and insufficient for $6N$ degrees of freedom \cite{takahashi2022kinematics,schweighart2006electromagnetic}. The third challenge is scalability: each satellite is influenced by unwanted magnetic disturbances from non-cooperative neighbors. The final challenge lies in the high computational cost of the exact magnetic field interaction model, which requires information on position, attitude, and coil geometry \cite{takahashi2025experimental,takahashi2025coil}. 

To overcome these constraints, previous studies proposed actuator-, system-, and controller-level solutions. An actuator-level solution is primarily a time-integrated control mechanism, where the average acceleration of actuator outputs matches the commanded value, a phenomenon commonly observed in insects and birds \cite{taha2020vibrational}. Their systematic approach is the alternating current (AC) method 
\cite{porter2014demonstration,youngquist2013alternating,abbasi2022decentralized}, adjusting both the phase and amplitude \cite{youngquist2013alternating,abbasi2022decentralized}. This decouples the coupling of the multiple satellites and Geomagnetic interaction on averaged dynamics \cite{youngquist2013alternating,abbasi2022decentralized} and realizes 6-DoF control for arbitrary satellite number \cite{takahashi2022kinematics,takahashi2021simultaneous}.
The rigorous interaction includes sinusoidal disturbances, which can be adjusted based on the frequency response analysis. Previous studies also proposed time-scheduled switching magnetic control \cite{ramirez2010new,takahashi2021simultaneous} and dipole polarity switching \cite{ahsun2006dynamics}, which periodically reverses the dipole direction to decouple from the Geomagnetic field. A system-level one uses additional attitude actuators, such as three-axis reaction wheels, with minimized electromagnetic torque \cite{schweighart2006electromagnetic}. Another path is controller-level one: a previous study \cite{takahashi2020time} shows EMFF inherently 
enables asymptotic state steering through non-smooth state feedback control. 
The purpose of reducing mass and maximizing mission efficiency and the accessibility of the Geomagnetic field push actuator-level solution makes it more accessible and attractive. 

Therefore, this study experimentally validates the principles of controlling large-scale satellite swarms using magnetic field interactions. We limit alternating magnetic field control for quantitative validation. Our contributions include designing the unified dipole-allocation framework, NODA-MMH, and the testbed to investigate key aspects of magnetically-acutuated swarm: (1) controllability extension, (2) decentralized current management, and (3) the effectiveness of a learning-based magnetic field interaction model, with a theoretically guaranteed error bound under control-induced disturbance and learning \& implementation errors.
\section{Preliminaries}
We begin by defining our mathematical notation. We write $\mathrm{range}(A)=\{Ax\mid x\in\mathbb{R}^n\}$, $S\in\mathrm{null}(A)$ for the null space of a matrix $A$, i.e., $AS=0$, and $x_N\triangleq[x_1;x_2;\ldots;x_n]$. We denote the Moore–Penrose pseudoinverse of a full row rank matrix $A\in\mathbb{R}^{m\times n_{\geq m}}$ by $A^+= A^\top(AA^\top)^{-1}$, $\mathrm{sym}(A)$ is symmetric part of $A$, i.e., $(A+A^\top)/2$. 
\subsection{Attitude Representation: Modified Rodrigues Parameters}
\label{MRP_definition}
Let $\{{a}\}^\top=\left\{{a}_x, {a}_y, {a}_z\right\}$ denote the basis vectors of a frame $\mathcal{A}$ as $\bm{p}=\{{a}\}^\top p^a=\{{b}\}^\top C^{B/A}p^a$ where $C^{B/A}\in\mathbb{R}^{3\times 3}$ is a coordinate transformation matrix from frame $\mathcal{A}$ to $\mathcal{B}$. The modified Rodrigues parameters $\sigma\in\mathbb{R}^3$ represent the attitude states \cite{wen1991attitude}. 
The kinematic equation of relative attitudes   (${\mathsf{e}}_{\sigma},\mathsf{e}_{\omega}^{b}$) is $\dot{\mathsf{e}}_{\sigma}=Z_{({\mathsf{e}}_{\sigma})} \mathsf{e}_{\omega}^{b}$ \cite{wen1991attitude} where $\mathsf{e}_{\omega}^{b}=\omega^{b}-\omega_{d}^{b}$ 
(see \cite{wen1991attitude,takahashi2022kinematics} about the definition of $Z_{\left(\sigma\right)}$, $C^{B/I}_{(\sigma)}$, and ${\mathsf{e}}_{\sigma}$).
\subsection{Multi-Layer Perceptron and Implementation Error}
\label{MLP_introduction}
The MLP model represents the functional mapping from inputs $x$ into outputs $\mathcal{F}_{(\bm{x}, \bm{\theta})}$ such as a ($L$+1)-layer neural network $\mathcal{F}_{(x, \bm{\theta})}= W^{L+1}\phi(\cdots\phi(W^1 x+b^1)\cdots)+b^{L+1}$ 
where the activation function $\phi(\cdot)$ and $\theta$ include the weights $\theta_w = W^1,\ldots, W^{L+1}$, the bias $\theta_b = b^1,\ldots, b^{L+1}$, and trained to minimize the user-defined loss-function. The Lipschitz constant of each layer is $\|W^l\mathbf{x}+b^l\|_{\text {Lip }}=\sup _{\mathbf{x}} \sigma(\nabla (W^l\mathbf{x}+b^l))=\sigma(W^l)$  where $\|f\|_{{\mathrm{Lip}}}$ is defined for a general differentiable function $f$ and $\forall {x}_1,{x}_2:
{\left\|f_{({x}_1)}-f_{\left({x}_2\right)}\right\|_2}/{\left\|{x}_1-{x}_2\right\|_2} \leq\|f\|_{\mathrm{Lip}}= \sup _{{x}} \sigma(\nabla f_{({x})})$. $\|\mathcal{F}\|_{\text {Lip}}$ is bounded as  
$\|\mathcal{F}\|_{\text {Lip}}\leq \overline{\mathcal{F}_{\text {Lip}}}\triangleq \|\phi\|_{\text {Lip}}^L\prod_{l=1}^{L+1} \sigma\left(W^l\right)$ 
\cite{miyato2018spectral} since $\left\|g_1 \circ g_2\right\|_{\text {Lip }} \leq\left\|g_1\right\|_{\text {Lip }} \cdot\left\|g_2\right\|_{\text {Lip }}$. 
Residual quantization encodes with low-bit integers and 
mitigates power, memory access, and errors, e.g., 4-bit quantization achieves FP16-level accuracy \cite{alvi2025deadline}. Soft errors from space radiation, such as bit flips, degrade $\|\mathcal{F}\|_{\text {Lip }}$. 
\subsection{Electromagnetic Interaction Model}
\label{EMFF}
This subsection presents the magnetic interaction models between circular air- or iron-core coils. We define $j$-th dipole moment of the multilayer coil $\bm{\mu}_{j}\triangleq\{\bm{b}\}^\top [\mu^b_{j_x};\mu^b_{j_y};\mu^b_{j_z}]$ where $\mu^b_{j_{v}}\triangleq N_{t}Ac_{j_{v}}(1+({\mu_r-1})/({1-N_d+\mu_r N_d}))$, 
$N_t$ is the number of coil turns, $A$ is the area enclosed by the coil, and $c_{j_{v}}$ is the current value. 
The Biot-Savart law \cite{schweighart2006electromagnetic} define 
the ``coil geometry vector'' \cite{takahashi2025coil}:
\begin{equation}
\label{circulant_integration_term}
{\bm{g}}_{j_{v}\leftarrow k_w}=
\begin{bmatrix}
\oint
\left[
\oint
\frac{{{r}_{j_{v}\leftarrow k_w}} \times {\mathrm{d}l}_{k_{w}}}{\|r_{j_{v}\leftarrow k_w}\|^3}\right] \times {\mathrm{d}l}_{j_{v}}\\
\oint
R_{j_{v}} \times\left[
\oint
\frac{{{r}_{j_{v}\leftarrow k_w}}\times {\mathrm{d}l}_{k_{w}}}{\|r_{j_{v}\leftarrow k_w}\|^3}\right] \times {\mathrm{d}l}_{j_{v}}
\end{bmatrix}
\in\mathbb{R}^{6}
\end{equation}
where $r_{j_{v}\leftarrow k_w}\in\mathbb{R}^3$ is the relative vector from the coil element ${\mathrm{d}l}_{k_w}$ to ${\mathrm{d}l}_{j_w}$. This introduces the $j$-th electromagnetic force $\bm{f}_{j\leftarrow k}$ and torque $\bm{\tau}_{j\leftarrow k}$  
and their total values on the $j$-satellites due to the neighbor satellites $k\in\mathcal{N}_j$ are \cite{takahashi2025coil}
\begin{equation}
\label{near_field_electromagnetic_interaction_model} 
\bm{u}_{j}=\begin{bmatrix}
    \bm{f}_{j}\\
    \bm{\tau}_{j}\end{bmatrix}
    =\sum_{k\in\mathcal{N}_j}
\bm{u}_{j\leftarrow k}
=\sum_{k\in\mathcal{N}_j}
\frac{\mu_0}{4\pi}G_{j\leftarrow k} \left(\mu^b_k\otimes \mu^b_j\right)
\end{equation}
where $\mu_0=4\pi e^{-7}$ T$\cdot$m/A is permeability constant, ${{G}}_{j\leftarrow k}=A^{-2}[[{\bm{g}}_{j\leftarrow k_x}],[{\bm{g}}_{j\leftarrow k_y}],[{\bm{g}}_{j\leftarrow k_z}]]\in\mathbb{R}^{6\times 9}$, $[{\bm{g}}_{j\leftarrow k_w}]$$=[{\bm{g}}_{j_{x}\leftarrow k_w\mathrm{off}},{\bm{g}}_{j_{y}\leftarrow k_w\mathrm{off}},{\bm{g}}_{j_{z}\leftarrow k_w\mathrm{off}}]\in\mathbb{R}^{6\times 3}$, and ${\bm{g}}_{j_{v}\leftarrow k_w\mathrm{off}}=[I_3,O_3;\left[r_{j_{v}\mathrm{off}}\right]_\times,I_3]{\bm{g}}_{j\leftarrow k}\in\mathbb{R}^{6}$. The far-field model \cite{ahsun2006dynamics,schweighart2006electromagnetic} is $\bm{u}_{j\leftarrow k}\approx[\nabla({\mu}_j \cdot {B}_k);\bm{\mu}_j \times B_k]$ when the coil radius is smaller than distance, i.e., $|R_{i}|\ll|r_{j_{v}\leftarrow k_w}|$.
\section{Survey: Problems and Proposed Solutions}
This section focuses on the problem formulation of controlling magnetic field interactions. Our analysis primarily focuses on AC control, which applies to the various time-integrated control strategies mentioned in the introduction. 
\subsection{Nonholonomic Constraint: Kinematics Control}
\label{Kinematics_Control}
We first mention the command values of electromagnetic force and torque to satisfy the conservation of angular momentum. 
For $n$ satellites and $m \in [1, n]$ satellites equipped with three-axis RWs, we define the tangent space 
\begin{equation}
    \label{tangent_space_kinematics_control}
S_{(n,m)} \in\text{null}(A_{(n,m)})
\end{equation}
where $A_{(n,m)}\in \mathbb{R}^{3 \times(6 n+3m)}$ is from angular momentum conservation $\sum_{j=1}^n(m_j\bm{r}_j \times \dot{\bm{r}}_j+\bm{I}_j \cdot \bm{\omega}_j)+\sum_{j=1}^m \bm{h}_j=\bm{L}$ 
$$
\Leftrightarrow A_{(n,m)}\zeta \triangleq
\begin{bmatrix}
&m_1 [r_{1}^i]_\times, \cdots, m_n [r_{n}^i]_\times,\\
&C^{I/ B_1} J_1, \cdots, C^{I/ B_n} J_n,\\
&C^{I/B_1}, \cdots, C^{I/ B_{m}}
\end{bmatrix}
\begin{bmatrix}
\dot{r}^{i}\\
\omega^{b}\\
\xi^{b}
\end{bmatrix}
=0.
$$
where 
$\xi_j^{b_j}=h_j^{b_j}- L^{b_j}/m$. 
Then, the lagrangian dynamics of $\zeta\in \mathbb{R}^{6n+3m-3}$ is $M\dot{\zeta}+C\zeta=B u+d-A_{(n,m)}^{\top} \eta$ \cite{takahashi2022kinematics} where constraint vector $\eta\in\mathbb{R}^3$, external inputs $d$, and
\begin{equation}
\label{Redundant_EMFF}
B,B^{-1}=
\begin{bmatrix}
E_{3 n} & O&O \\
O& E_{3 n} & \mp E_{3 n \times 3m}\\
O &O & E_{3 m}
\end{bmatrix},\quad
u=
\begin{bmatrix}
    f_c^{i}\\
    \tau_c^{b}\\
    \dot{h}^{b}
\end{bmatrix}
\end{equation}
Substituting $\zeta=S v$ and $\dot{\zeta}=\dot{S} v+S \dot{v}$ and multiplying by $S^{\top}$ from left yields
the kinematics of EMFF \cite{takahashi2022kinematics}:
\begin{equation}
\label{EMFF_dynamics_null_space}
\begin{aligned}
\dot{q}
=P_{({\sigma_j})}\zeta,\quad\quad \overline{M}_{\left(q_s\right)} \dot{v}+\overline{C}_{\left(q_s, v\right)} v=S^{\top}({B}u+d)
\end{aligned}
\end{equation} 
where $\dot{q}
=[\dot{r}^{i}_N;\dot\sigma_N]$, $\overline{M}=S^{\top}M S$, $\overline{C}=S^{\top}(M\dot{S}+CS)$,  
$$
P_{({\sigma_j})}\triangleq
\begin{bmatrix}
    I&O&O\\
    O&Z_{({\mathsf{e}}_{\sigma_1})}\oplus\cdots\oplus Z_{({\mathsf{e}}_{\sigma_n})}&O
\end{bmatrix}
$$
\begin{theorem}[Kinematics control \cite{takahashi2022kinematics}]
\label{Theorem_Kinematics_Control}
The magnetic interaction can realizes $u_c
\in\mathrm{range}(B^{-1}MS)$, i.e., $AM^{-1}Bu_c=0$ 
and applying $u_c=\overline{B}_r^{-1}(-K \mathsf{e}_v-{\overline{d}}+\overline{M} \dot{v}_d+\overline{C} v_d)$ to Eq.~(\ref{Redundant_EMFF}) achieves $q \rightarrow q_{d}$ as $t \rightarrow \infty$ where $S^\top B\overline{B}_r^{-1}=I$, $\mathsf{e}_v$ is composite variable \cite{slotine1991applied}, e.g., Eq.~(\ref{composite_variable}).
\end{theorem}
\subsection{Underactuation and Scalability: Time-Integrated Control}
\label{Time_Integrated_Control_of_Alternating_Current}
This subsection introduces a time-integration current control, particularly the AC method, to enhance controllability and address scalability issues. 
We rewrite the model in Eq.~(\ref{near_field_electromagnetic_interaction_model}) as a multivariate bilinear polynomial system \cite{schweighart2006electromagnetic,takahashi2022kinematics}
$$
\begin{aligned}
&\text{s.t.\ }\left\{
\begin{aligned}
&
\underbrace{f_{cj(x,y,z)},\tau_{cj(x,y,z)}}_{\substack{\text{= Command input}}}=\sum_{k\neq j}\left\{f_{j\leftarrow k(x,y,z)},\tau_{j\leftarrow k(x,y,z)}\right\}\\
&=
\sum_{k=1}^3 \sum_{\substack{l=1 \\(l \neq j)}}^n \sum_{m=1}^3 \underbrace{\mu_{j k}}_{\text{Variable}} \underbrace{C_{jklm}}_{\text{Const.}} \underbrace{\mu_{l m}}_{\text{Variable}}\quad j\in[1,N]
\end{aligned}
\right.
\end{aligned}
$$
where $C_{jklm}\in\mathbb{R}$ is a constant associated with system state. This involves $3n$ variables and $6n$-$6$ polynomial equations. Consequently, a DC-based magnetic field can generate arbitrary electromagnetic forces for $n$ satellites, but simultaneous control is only possible for up to two satellites. 
We assume that the MTQs of the $j$th satellite are driven by sinusoidal signals with angular frequency $\omega_j$ [rad/s]: ${\mu}_j(t)=s_{j}\sin \left(\omega_j t\right)+c_{j} \cos \left(\omega_j t\right)$ for $j \in [1, n]$ and $u_{j\leftarrow k(t)}=\sum_{k\neq j}$$\frac{\mu_0}{8\pi}G_{j\leftarrow k}$${\sum_{\small{(\mathrm{a},\mathrm{b})=\{(\mathrm{s},\mathrm{s}),(\mathrm{c},\mathrm{c}),(\mathrm{s},\mathrm{c}),(\mathrm{c},\mathrm{s})\}}}\mathrm{a}_{{\omega_k t}}\mathrm{b}_{{\omega_j t}}\mathrm{a}_k^b\otimes \mathrm{b}_j^b}$.
The agents operating at different frequencies do not interact with each other in the first-order averaged dynamics, i.e.,
$$
\int_0^T \sin(\omega_j\tau)\sin(\omega_k\tau)\frac{\mathrm{d}\tau}{T}=
\left\{
\begin{aligned}
    &0\ &&(\mathrm{if}\ &&\omega_j\neq \omega_k)\\
    &\frac{1}{2}\ &&(\mathrm{else\ if}\ &&\omega_j= \omega_k)
\end{aligned}
\right..
$$
We obtain the averaged interaction $u_{j\leftarrow k}^\mathrm{avg}=
\int_0^T u_{j\leftarrow k(\tau)}\frac{\mathrm{d}\tau}{T}$:
\begin{equation}
\label{averaged_input}
u_{j}^\mathrm{avg}=\sum_{k\in\mathcal{N}_j}u_{j\leftarrow k}^\mathrm{avg}
\approx\frac{1}{2}
\frac{\mu_0}{4\pi}G_{j\leftarrow k} \left(s^b_k\otimes s^b_j+c^b_k\otimes c^b_j\right)
\end{equation}
Then, the AC-based optimal allocation problem $\mathcal{OPT}_{\mathrm{AC}}$ for a given evaluation function $J$ with arbitrary parameter $\chi$ is
$$
\begin{aligned}
&\mathcal{OPT}_{\mathrm{AC}}:\ \text{min.\ }J\left(s_{1},\ldots,s_{n},c_{1},\ldots,c_{n},\chi\right)\\
&\text{s.t.}
\left\{
\begin{aligned}
&\underbrace{f_{cj(x,y,z)},\tau_{cj(x,y,z)}}_{\substack{\text{= Const. command input}}}= \sum_{k\neq j}\left\{f^{\mathrm{avg}}_{j\leftarrow k(x,y,z)},\tau^{\mathrm{avg}}_{j\leftarrow k(x,y,z)}\right\}\\
&=\sum_{k=1}^3 \sum_{\substack{l=1 \\
(l \neq j)}} \sum_{m=1}^3\left(\underbrace{s_{j k}}_{\mathrm{var.}} \underbrace{C^{\mathrm{sin}}_{jklm}}_{\mathrm{const.}} \underbrace{s_{lm}}_{\mathrm{var.}}+ \underbrace{c_{jk}}_{\mathrm{var.}} \underbrace{C^{\mathrm{cos}}_{jklm}}_{\mathrm{const.}} \underbrace{c_{lm}}_{\mathrm{var.}}\right)
\end{aligned}
\right.
\end{aligned}
$$
for $j\in[1,N]$ where $C^{\mathrm{sin},\mathrm{cos}}_{jklm}\in\mathbb{R}$ are constants of the system state. Since this includes $6n$ variables and $6n$-6 polynomials,
the equality constraints in $\mathcal{OPT}_{\mathrm{AC}}$ potentially yield \( 2^{6n-6} \) solutions and six free variables, and simultaneous control can be achieved for $n$ satellites \cite{takahashi2022kinematics}. Notably, the lower bound of $\mathcal{OPT}_{\mathrm{AC}}$ is gained by its Lagrange dual problem \cite{takahashi2024neural} with $P_{(k_{\lambda},\hat{\lambda}_N)}$ associated with $C^{\mathrm{sin},\mathrm{cos}}_{jklm}$ \cite{takahashi2024neural}:
\begin{equation}
\label{dual_problem}
\mathcal{OPT}^{\mathrm{Dual}}_{\mathrm{AC}}
\max_{\hat{\lambda}_N^* \in \mathbb{R}^{6(n-1)}} : -\frac{\hat{\lambda}_N^{\top} {u}^a}{\mu_0/8\pi}\ \ \text{s.t.:}\ \ P_{(k_{\lambda},\hat{\lambda}_N)} \succeq O.
\end{equation}
\subsection{Computational Cost: Learning-Aided Dipole Allocation}
The calculations of $\mathcal{OPT}_{\mathrm{AC}}$ and ${\bm{g}}_{j_{v}\leftarrow k_w}$ in Eq.~(\ref{circulant_integration_term}) take computational costs and previous studies successfully approximate these two types of MLP: power-optimal control allocation for arbitrary satellite numbers \cite{takahashi2024neural}
\begin{equation}
\label{NODA_model_expression}
s_{lm},\ c_{lm}=\mathfrak{m}(\hat{r}^{l},\hat{f}^{l},\hat{\tau}^{l},\bm{\theta})
\quad\mathrm{s.t.:}\quad\mathcal{OPT}_{\mathrm{AC}}
\end{equation}
and a learning-based exact magnetic field model \cite{takahashi2025experimental,takahashi2025coil}
\begin{equation}
\label{coil_geometry_model_expression}
{g}_{j_{v}\leftarrow k_w}
=\mathfrak{g}{\left(\bm{r}^{(i)}_{[2]}/\gamma,\bm{\phi}^{(i)}_{[2]}, \bm{\theta}\right)}
\begin{bmatrix}
    I&O\\
    O&\gamma I
\end{bmatrix}
\end{equation}
\section{Main Result: power-optimal nonlinear control for 
magnetically actuated swarm}
\label{nonlinear_robust_controller}
This section proposes a learning-based dipole allocation strategy certified by a safety-stability analysis. Based on the preceding discussion, the magnetically actuated swarm system can be formulated as a multi-agent system subject to nonholonomic constraints and control-induced disturbance. For this system, this study takes a model-based controller design approach enabled by the learned magnetic field model. Our NODA-MMH provides the desired coil currents to minimize tracking error while achieving optimal power consumption utilizing the learned model of  $\mathcal{OPT}_{\mathrm{AC}}$ in subsection~\ref{Time_Integrated_Control_of_Alternating_Current} for a given neighbor satellite information.
\subsection{Control-Induced Disturbance Attenuation}
We first prove that power-optimal dipole allocation also leads to the attenuation of time-integrated control-induced disturbances. For quantitative analysis, we employ the AC magnetic control of $\omega$ and the rigorous magnetic interaction includes sinusoidal disturbances ${u}_{j\leftarrow k(t)}=u_{j\leftarrow k}^{\mathrm{avg}}+{d}_{j\leftarrow k(t)}^{2\omega}$ where $u_{j}^{\mathrm{avg}}$ in Eq.~(\ref{averaged_input}). We analytically derive ${d}_{j(t)}^{2\omega}=\{\mathrm{c}_{{2\omega t}}x+\mathrm{s}_{{2\omega t}}y\}$ where $[x;y]=\sum_{k\in\mathcal{N}_j}\frac{\mu_0}{8\pi}(I_2\otimes G_{j\leftarrow k})z$ and $z=[c_k^b\otimes c_j^b-s_k^b\otimes s_j^b;c_k^b\otimes s_j^b+s_k^b\otimes c_j^b]$.
\begin{lemma}
\label{quadratic_upper_bound}
Power $w$ linearly bounds $\sup_{t\in[0,T)}\|{d}_{j(t)}^{2\omega}\|$.
\end{lemma}
\begin{proof}
Since $\|{d}_{j}^{2\omega}\|^2
=v_{(t)}^\top
[\|x\|^2,x^\top y;x^\top y,\|y\|^2]
v_{(t)}$ with $v_{(t)}\triangleq 
[\mathrm{s}_{{2\omega t}};\mathrm{c}_{{2\omega t}}]$ yieds $\sup_{t\in[0,T)}\|{d}_{j\leftarrow k}^{2\omega}\|$
\begin{equation}
\label{result}
=\sqrt{{\left\|
[x;y]\right\|^2/2+\sqrt{(\|x\|^2-\|y\|^2)^2+4(x^\top y)^2}}/{2}}
\end{equation}
where we use $\|v\|=1$.  
We define $W_0$ and $W=W^\top\succ 0$ such that $z=W_0(m\otimes m)$ and $0 \preceq W_0^\top (I\otimes G_{j\leftarrow k})^\top (I\otimes G_{j\leftarrow k})W_0 \preceq  (W\otimes W)$. Then, we obtain $\sup_{t\in[0,T)}\|{d}_{j\leftarrow k}^{2\omega}\|^2\leq\|[x;y]\|^2\leq
(\frac{\mu_0}{8\pi}m^\top W m)^2$
where we apply $\lambda_{\max(\cdot)}\leq \mathrm{Trace}_{(\cdot)}$ and the Kronecker identity.
\end{proof}
\noindent
Minimizing $w$ for agents operating $\omega_j$ also leads to reducing the control-induced disturbance from an uncooperative neighbor using $\omega_k$, although we omit a proof for simplicity.
\subsection{Attitude-Term Decoupled Dipole Allocation}
For proximity operation, the power-optimal dipole allocation problem $\mathcal{OPT}_{\mathrm{AC}}$ in subsection~\ref{Time_Integrated_Control_of_Alternating_Current} naturally includes the nonlinear coil geometry information. Then, this subsection decouples its attitude-induced terms using the coil geometry learned model $\mathfrak{g}$ to reduce the sample region. 
Let $X_{\zeta}\in\mathbb{R}^{6(n-1)\times 9n^2}$ be stack matrix for arbitrary matrix $X_{j\leftarrow k}\in\mathbb{R}^{6\times 9}$ and its formulation for three satellite is
$$
X_{\zeta}\triangleq
K
\begin{bmatrix}
    X_{1\leftarrow 2}&O&X_{1\leftarrow 3}K_{3,3}\\
    X_{2\leftarrow 1} K_{3,3}&X_{2\leftarrow 3}&O 
\end{bmatrix}
R
$$
where $K\in \mathbb{R}^{6(n-1)\times 6(n-1)}$ is a block permutation matrix that reorders $[f_1,\tau_1,f_2,\tau_2,\ldots]$ into $[f_1,f_2,\ldots,\tau_1,\tau_2,\ldots]$ and $K_{3,3}$ is the constant matrix such that $K_{3,3}(a \otimes b) = (b \otimes a)$, $R=[E_2 \otimes E_1;E_3 \otimes E_2;\ldots;E_1 \otimes E_n]\in\mathbb{R}^{(9n(n-1)/2)\times 9n^2}$, $E_i=(e_i^\top \otimes I_3)\in\mathbb{R}^{3\times 3n}$ and $e_i$ is the $n$ dimentional basis. This express the exact/far-model magnetic interactions using $[\mu]=\left(s_N \otimes s_N + c_N \otimes c_N\right)\in\mathbb{R}^{9n^2}$:
\begin{equation}
\label{exact_input_model}
u_{\mathrm{exa}}\triangleq
\begin{bmatrix}
    f_{\mathrm{exa}}\\
    \tau_{\mathrm{exa}}
\end{bmatrix}=\frac{1}{2}\frac{\mu_0}{4\pi}
G_\zeta[\mu],\quad  
u_{\mathrm{far}}
\triangleq\frac{1}{2}\frac{\mu_0}{4\pi}
Q_{\zeta}[\mu]
\end{equation}
\begin{lemma}[Exact/far-magnetic model mapping]
Consider $[\mu]\in\mathrm{range}(P_\zeta^\top)$ using arbitrary $P_\zeta\in\mathbb{R}^{6n\times 9n}$ such that $\mathrm{rank}(P_\zeta)= 6n$. Then, we obtain $u_{\mathrm{exa}}=H_\zeta u_{\mathrm{far}}$ where $H_\zeta \triangleq \{H_\zeta \in \mathbb{R}^{6n-3\times 6n} \mid G_\zeta P_\zeta^\top = H_\zeta Q_\zeta P_\zeta^\top 
\}$.
\end{lemma}
\begin{proof}
Substituion $[\mu]$ and $G_\zeta P_\zeta^\top$ into $G_\zeta[\mu]$ concludes. 
\end{proof}
\noindent
We define $s$ as the dimension number of the kernel space of $P_\zeta$. If the $2s^2$ is smaller than $6(n-1)$, there exists $[\mu]\triangleq P_\zeta^\top \eta$ that satisfies Karush-Kuhn-Tucker condition conditions \cite{takahashi2024neural}. 
\noindent
We define the orthogonal projector $\Pi_{Q_\zeta}\triangleq Q_\zeta^+ Q_\zeta$ 
such that $(I-\Pi_{Q_\zeta})G_\zeta^{\top}=O$ if $\mathrm{null}(G_\zeta^\top)\subseteq \mathrm{null}(Q_\zeta^\top)$. A trivial subset of the common null space of $Q_\zeta^\top$ and $G_\zeta^\top$ is given by 
Eq.~(\ref{tangent_space_kinematics_control}) and 
we can newly define for $\mathrm{null}(G_\zeta^\top)\supseteq \mathrm{null}(Q_\zeta^\top)$:
\begin{equation}
\label{new_tangent_space}
S_{(n,m)}\in
\mathrm{null}(A_{(n,m)})\cap\mathrm{null}\big((I-\Pi_{Q_\zeta})G_\zeta^\top B^\top\big),
\end{equation}
along with the definition in Eq.~(\ref{tangent_space_kinematics_control}). This decouples the 
\begin{equation}
\label{decoupling_attitude_matrix}
S^\top B \begin{bmatrix}
    u_{\mathrm{exa}}\\
    \dot{h}
\end{bmatrix}
\triangleq S^\top B_{\mathrm{e/f}}
\begin{bmatrix}
    u_{\mathrm{far}}\\
    \dot{h}
\end{bmatrix}
,\ B_{\mathrm{e/f}}\triangleq B\begin{bmatrix}G_\zeta Q_\zeta^+&O\\O&I\end{bmatrix}
\end{equation}
\subsection{Model-based Optimal Nonlinear Controller Design}
Finally, we present a convex optimization-based nonlinear controller for optimal submanifold stabilization with safety-stability analysis. In contrast to the asymptotic stability in Theorem~\ref{Theorem_Kinematics_Control}, we design the dipole moment allocation for exponential stability that can be directly derived by learning-based dipole allocation $\mathfrak{m}$ in Eq.~(\ref{NODA_model_expression}) with  $\mathfrak{g}$ in Eq.~(\ref{coil_geometry_model_expression}). 
\begin{theorem}[Input-to-State Stability Analysis]
\label{theorem_NODA_MMH}
For given $\zeta_d\triangleq[\dot{r}_d^{i};C^{B_j / B_{j d}} \omega_{j d}^{b_{j_d}};0_{3m}]$, let $\mathsf{e}_v$ be defined with $\Lambda\succ O$:
\begin{equation}
\label{composite_variable}
\mathsf{e}_v=v-v_r,\quad v_r \triangleq (P_{(\mathsf{e}_{\sigma})}S)^+\left(P_{(\mathsf{e}_{\sigma})}\zeta_d -\Lambda\mathsf{e}_q\right)
\end{equation}
where $P_{(\cdot)}$ in Eq.~(\ref{EMFF_dynamics_null_space}).
Consider the nonlinear controller:
\begin{equation}
\label{designed_input}
\begin{bmatrix}
u_{\mathrm{far}}\\
    \dot{h}
\end{bmatrix}
=u_r -KS\mathsf{e}_v,\ u_r=B_{\mathrm{e/f}}^{-1}\left(M\dot{\zeta}_r+C\zeta_r -\hat{d}\right)\\
\end{equation}
where $\hat{{d}}$ is the estimation values of $d$, $S_{(n,m)}$ is in Eq.~(\ref{new_tangent_space}), $\zeta_r=Sv_r$, and the gain matrix $K$ satisfies
\begin{equation}
\label{inequality_contraction}
-S^\top \mathrm{sym}(B_{\mathrm{e/f}}K) S\preceq -\alpha S^\top M S.
\end{equation}
Then, the state error ${\mathsf{e}}_q$ in the closed-loop system applied the dipole moment $\{[\mu]|\frac{\mu_0}{8\pi}Q_\zeta[\mu]=u_{\mathrm{far}}\mathrm{\ in\ Eq.~(\ref{designed_input})}\}$ to Eq.~(\ref{Redundant_EMFF}) is globally exponentially converged into the error ball: 
\begin{equation}
    \label{steady_error}
\lim_{t\rightarrow\infty}\|{\mathsf{e}}_{q(t)}\|\leq \frac{\sup_{\zeta,t}\|P_{(\mathsf{e}_{\sigma})}\|_2}{\sqrt{\lambda_{\min}(M)\ \lambda_{\min}(\Lambda)}}\frac{\sup_{\zeta,t}\|d\|}{\alpha}.
\end{equation}
\end{theorem}
\begin{proof}
See Appendix~\ref{proof_NODA_MMH} or \cite{liu2020robust,tsukamoto2021contraction} for state analysis.
\end{proof}
\noindent
Lemma~\ref{quadratic_upper_bound} guarantees power optimal allocation minimize the upper-bound of $\sup_{t}\|{d}_{j}^{2\omega}\|$. We minimize the upper bound in Eq.~(\ref{steady_error}) by maximizing $\alpha$ through optimization of $S_{(n,m)}$ and $K$ under Eq.~(\ref{inequality_contraction}). This study limits $S_{(n,m)}=\mathcal{S}_0$ in 
\cite{takahashi2022kinematics} to convexify Eq.~(\ref{inequality_contraction}) where $
\mathcal{S}_{0}\triangleq[I;-C^{B_{m}/I}A_{(:,\mathrm{end}-3:\mathrm{end})}]
$ and $v=[\dot{r}^{i};\omega^{b};\xi_{1:m-1}^{b}]$. 
\begin{remark}
If the far-field approximation is valid, i.e., $G_\zeta = Q_\zeta$, the control gain mattrix $K\triangleq K_0+(G_\zeta Q_\zeta^+-I)K$ with a given constant $K_0$ reduces to linear control. This assumes that the higher current-update frequency relative to position changes can remove the nonlinearity of $Q_\zeta$ in this region. This constant gain strategy is suited for reliable connectivity maintenance \cite{takahashi2025distance,takahashi2025anticipatory}, which mitigates the curse of dimensionality problem in spacecraft swarm deployment. 
\end{remark}
\begin{theorem}[NODA-MMH]
For a given upper limit $\overline{w}$ [W], the following convex optimization to minimize $\alpha_{\mathrm{inv}}\triangleq\alpha^{-1}$ derives the sub-optimal gain matrix $K^*$ to minimize tracking error while achieving optimal power consumption:
\begin{equation*}
\label{convex_programming}
\begin{aligned}
&K^*= \argmin_{\substack{K\in\mathbb{P}^{6N\times 6N},\ \hat{\lambda}\in\mathbb{R}^{6(N-1)}, \alpha_{\mathrm{inv}}>0,\ \varepsilon>0,\ \overline{u}>0}}\  \alpha_{\mathrm{inv}} \\
&s.t.\ \left\{
\begin{aligned}
&AM^{-1}BK\mathcal{S}_0 =0,\quad\|\hat{\lambda}\|\leq \overline{u},\quad \|K\mathcal{S}_0\mathsf{e}_v\|\leq \overline{u}\\
& \begin{bmatrix}
\mathcal{S}_0^\top\mathrm{sym}(B_{\mathrm{e/f}}K)\mathcal{S}_0 &\mathcal{S}_0^\top\\
\mathcal{S}_0 & \alpha_{\mathrm{inv}}M^{-1}
\end{bmatrix}\succeq 0\\
&aI_{3 n}+O_{\hat{\lambda}_N}+O_{\hat{\lambda}_N}^{\top}
\succeq 0,\ 
\begin{bmatrix}
\left(\frac{\mu_0}{8\pi}\overline{w} \right)\varepsilon +\hat{\lambda}_N^{\top}u_r&\overline{u}\\
\overline{u}&1\\
\end{bmatrix}\succeq 0
\end{aligned}
\right .\\
\end{aligned}
\end{equation*}
if the duality gap between $\mathcal{OPT}_{\mathrm{AC}}$ and $\mathcal{OPT}^{\mathrm{Dual}}_{\mathrm{AC}}$ is zero.
\end{theorem}
\begin{proof}
The primal objective $w_{\mathrm{true}}$ is equal to the dual lower bound, and we relax the nonconvex constraint $w_{\mathrm{true}}\leq \overline{w}$ to 
$\frac{\mu_0}{8\pi} w_{\mathrm{true}}=-{\lambda}_N^{\top} (u_r -K\mathcal{S}_0\mathsf{e}_v)\leq -{\lambda}_N^{\top}u_r +\frac{\varepsilon\|{\lambda}_N\|^2}{2} +\frac{\|K\mathcal{S}_0\mathsf{e}_v\|^2}{2\varepsilon}$ 
where we use $0\leq \|\sqrt{\varepsilon} x-\sqrt{\varepsilon}^{-1}y\|^2 \Leftrightarrow 2x^\top y \leq \varepsilon\|x\|^2 + \varepsilon^{-1}\|y\|^2$.
The quadratic upper bound is tightest when $\varepsilon^\star=\overline{u}/\overline{\lambda}$ with $\|{\lambda}_N\|\leq \overline{\lambda}$ and $\|K\mathcal{S}_0\mathsf{e}_v\|\leq \overline{u}$. Then, multiplying $\varepsilon^\star$ above inequality and introducing $\hat{\lambda}=\varepsilon^\star \lambda$ yield 
$
-\hat{\lambda}_N^{\top}u_r +\overline{u}^2\leq \frac{\mu_0}{8\pi} \varepsilon^\star\overline{w}$, $\|\hat{\lambda}\|\leq \overline{u}$, and $\|K\mathcal{S}_0\mathsf{e}_v\|\leq \overline{u}$
Applying the Schur complements \cite{boyd1994linear} derives results.
\end{proof}
\section{Experimental Evaluation of Time-Integrated Magnetic Control Toward Orbit Proof}
This section validates the critical aspects of time-interated magnetic control through the experimental setups in Fig.~\ref{fig:Gravity_Compensation_Experiment}.
\subsection{Experimental Setup and Coil Design in Appendix~\ref{Testbed_Design_and_Gravity_Compensation_Test}}
We briefly summarize the experimental setup and custom-made MTQ design resutl in Appendix~\ref{Testbed_Design_and_Gravity_Compensation_Test}. The testbed consists of a 1.5m linear air track (Eisco, PH0362A) with an air blower (PH0363A) and a single-axis air bearing. Relative distance and absolute attitude are measured using a ToF sensor (VL53L0X) and an AR marker, respectively. 
MTQs are mainly traveling along the track, while one MTQ is mounted on the air bearing for frictionless single-axis rotation. The time synchronization for time-integrated control in subsection~\ref{Time_Integrated_Control_of_Alternating_Current} is achieved by the Pulse Per Second signal of GPS, and its accuracy is 0.1ms. 
We limit the input vector, aligning with the movable directions to enable partial validation of 6-DoF control feasibility. The fundamental control experiments by $u_{\mathrm{far}}$ in Eq.~(\ref{exact_input_model}) yield stable results around the target state as shown in Fig.~\ref{fig:Gravity_Compensation_Experiment}. 
\begin{figure}[tb]
  \centering
  \hspace*{-0.3cm} 
  \begin{minipage}[b]{\columnwidth}
    \centering
    \subfloat[Setups for distance, position, and angle control.]{
      \includegraphics[width=0.9\linewidth]{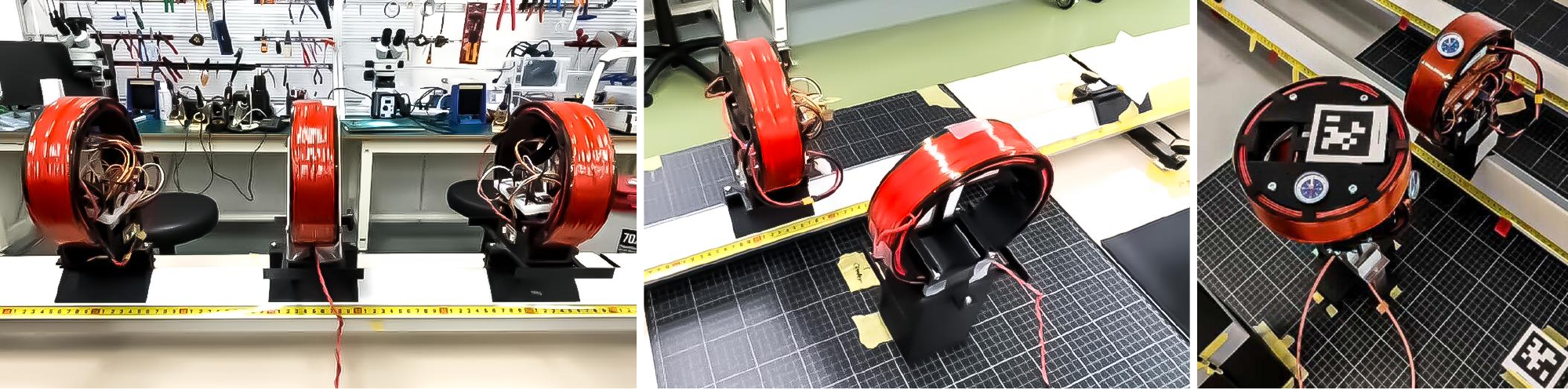}
      \label{experimant_setup_picture}
    }
  \end{minipage}\\ 
\hspace*{-0.3cm} 
  \begin{minipage}[b]{0.245\columnwidth}
    \centering
    \subfloat[Distance control.]{
      \includegraphics[width=1.025\linewidth]{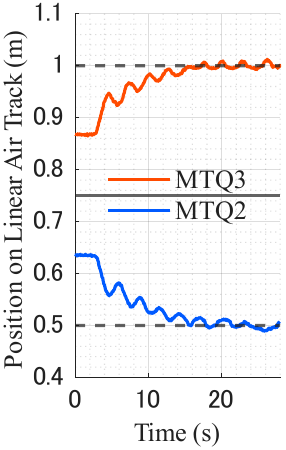}
      \label{fig:token_1d}
    }
  \end{minipage}
  \begin{minipage}[b]{0.415\columnwidth}
    \centering
    \subfloat[Position control.]{
      \includegraphics[width=1.025\linewidth]{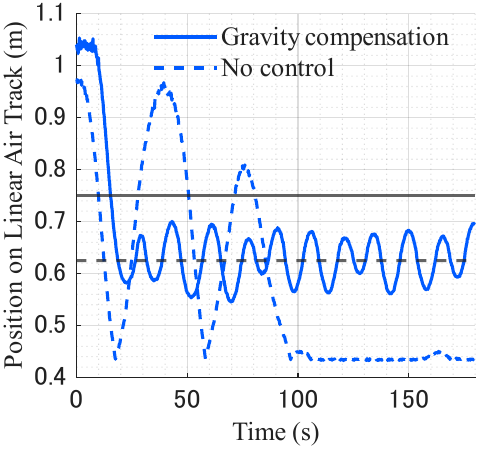}
      \label{fig:token_2d}
    }
  \end{minipage}
  \begin{minipage}[b]{0.305\columnwidth}
    \centering
    \subfloat[Angle control.]{
      \includegraphics[width=1.05\linewidth]{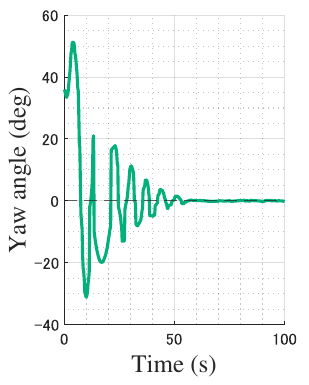}
      \label{fig:token_2d}
    }
  \end{minipage}
  \caption{Testbeds and gravity compensation results.}
  \label{fig:Gravity_Compensation_Experiment}
\end{figure}
\subsection{Unstable Dynamics Normalization for Orbital Dynamics}
\label{Evaluation_Criteria_Toward_Orbit_Proof}
We next outline the evaluation criteria for the ground experiment to partially ensure dynamic equivalence between the ground testbed and the on-orbit system. For a given constant target position $p_d\in\mathbb{R}^N$, 
we define the error states as $p_e=p-p_d$ and $v_e=v-v_d$. Given a fixed interaction graph $\mathcal{G}$ and the associated incidence matrix $E$, the edge states of the agents in the ground experiment are defined as $\{\mathsf{e}_{p_e},\mathsf{e}_{v_e}\}=E^\top \{p_e,v_e\}$, and $\dot{\mathsf{e}}_{v_{e}}=E^\top(u+d)$. For given $u=-K_p{\mathsf{e}}_{p_{e}}-K_d{\mathsf{e}}_{v_{e}}$, its closed-loop system is 
\begin{equation}
\label{ground_dynamics}
\begin{bmatrix}
    \dot{\mathsf{e}}_{v_{e}}\\
    \dot{\mathsf{e}}_{p_{e}}
\end{bmatrix}
\triangleq
\underbrace{\begin{bmatrix}
    -E^\top K_{d}&-E^\top K_p\\
    I&O\\
\end{bmatrix}}_{A_{\mathrm{gnd}}}
\begin{bmatrix}
    {\mathsf{e}}_{v_{e}}\\
    {\mathsf{e}}_{p_{e}}
\end{bmatrix}
+
\begin{bmatrix}
    E^\top d\\
    0    
\end{bmatrix}.
\end{equation}

To address MTQ control degradation by unstable part of relative orbital dynamics, we consider  the feedback control \cite{takahashi2025distance} to attenuate the drift terms $\mathsf{e}_{1}=E^\top (-2C_1)$ and $\mathsf{e}_{4}=E^\top C_4$ in Eq.~(\ref{relative_orbital_dynamics}) and the closed-loop system is 
\begin{equation*}
\label{relative_orbital_dynamics}
\begin{bmatrix}
    \dot{\mathsf{e}}_{1}\\
    \dot{\mathsf{e}}_{4}
\end{bmatrix}=
\underbrace{\begin{bmatrix}
   A_{11}&O\\
     \frac{\epsilon_2}{2}\left(I    
     +\frac{k_1}{k_A}A_{22} \right )&A_{22}
\end{bmatrix}}_{A_{\text{orb}}}
\begin{bmatrix}
    {\mathsf{e}}_{1}\\
    {\mathsf{e}}_{4}
\end{bmatrix}
-k_0
\begin{bmatrix}
    E^\top {D_y}\\
    0
\end{bmatrix}
\end{equation*}
where $A_{11}=-({k_A}/{2})L_e$, $A_{22}=\gamma A_{11}$, $k_0
\approx 1.8e^{3}$. 
We consider a coordinate transformation $\Theta\in\mathbb{R}^{2n\times 2n}$:
\begin{equation}
\label{coordination_matrix}
\begin{bmatrix}
\mathsf{e}_1\\
\mathsf{e}_4
\end{bmatrix}
\triangleq
\begin{bmatrix}
\Theta_{11} & \Theta_{12}\\
O & \Theta_{22}
\end{bmatrix}
\begin{bmatrix}
    {\mathsf{e}}_{v_{e}}\\
    {\mathsf{e}}_{p_{e}}
\end{bmatrix}\quad\mathrm{s.t.}\quad\Theta A_{\text{gnd}}=\beta A_{\text{orb}}\Theta
\end{equation}
where $A_{\text{orb}}$ is from Appendix~\ref{relative_orbital_dynamics} and \cite{takahashi2025distance}, 
$\beta$ is the disturbance ratio $\beta\approx {\|\Theta_{11} E^\top d\|_{\infty}}/{\|k_0 E^\top D_y\|_\infty}$ and $\tau \triangleq \beta t$.
\begin{theorem}[Unstable Dynamics Normalization]
\label{theorem_Dynamics_Normalization}
$\Theta$ in Eq.~(\ref{coordination_matrix}) exists for the closed-loop system applied 
$f_c=-k_v L\left(v-v_r\right)$ where 
$k_v \triangleq \beta k_A(1+\gamma )/{2}$, $v_r\triangleq v_d-\frac{\beta k_A\gamma}{2(1+\gamma)} (L p_e)\triangleq v_d-\frac{k_p}{k_v} (L p_e)$: $\Theta_{11}=\sum_{m=0}^{r}\alpha_m L_e^{m}$ and 
\begin{equation}
    \label{Theta_11}
\Theta_{12}
=\frac{\beta k_A\gamma}{2}L_e \Theta_{11},\ \Theta_{22}=\frac{\beta \epsilon_2}{2}\left(I    
     -\frac{\gamma k_1}{2}L_e \right )\Theta_{11}
\end{equation} 
\end{theorem}
\begin{proof}
See Appendix~\ref{proof_Dynamics_Normalization}.
\end{proof}
\noindent
Multiplying $\Theta$ and $\beta^{-1}$ on the left side of Eq.~(\ref{ground_dynamics}) derives
\begin{equation*}
\begin{aligned}
\frac{\mathrm{d}}{\mathrm{d}(\beta t)}
\begin{bmatrix}
\mathsf{e}_1\\
\mathsf{e}_4
\end{bmatrix}
&=
A_{\text{orb}}
\begin{bmatrix}
\mathsf{e}_1\\
\mathsf{e}_4
\end{bmatrix}
+\frac{1}{\beta}
\begin{bmatrix}
   \Theta_{11} E^\top d\\
    0    
\end{bmatrix}
\end{aligned}
\end{equation*}
The timescale transformation $t_{\mathrm{gnd}} = t_{\mathrm{orb}}/\beta$ adjusts the command update period $T_s$, whereas $\Theta$ provides a one-to-one mapping between ground- and orbital-frame errors. 
\subsection{Experimental Evaluation of Learned Magnetic Model}
As a fundamental part of NODA-MMH, we investigate the effectiveness of the learning-based magnetic field model $\mathfrak{g}$ in Eq.~(\ref{coil_geometry_model_expression}) through experiments. We design the PID controller as command values of force and torque, and implement two dipole allocations to achieve command through magnetic field interactions: the dipole approximation model and the learned model, referred to as "inv" and "cgl", respectively. In our experiment, to compensate for the microgravity effect indicated by the black line in Fig.~\ref{fig_proximately_control_1_2}, the magnetic field model is inferred to update coil current every $\mathrm{d}t_{\mathrm{gnd}}=187.5$ ms.
\subsubsection{Theoretical Bound Verification for Position Control}
We first conduct position control for far-range and proximity operation to compare the two dipole allocations. Through experiments, the steady state error satisfies the error bound in Eq.~(\ref{result}) and Appendix~\ref{Proof_Error_Bound_Control} due to time-integrated control-induced disturbances ${d}_{j\leftarrow k}^{2\omega}$ that is indicated by the dashed line.
\begin{figure}[tb]
\centering
\begin{minipage}[b]{0.49\columnwidth}
\centering
\includegraphics[width=1.0\linewidth]{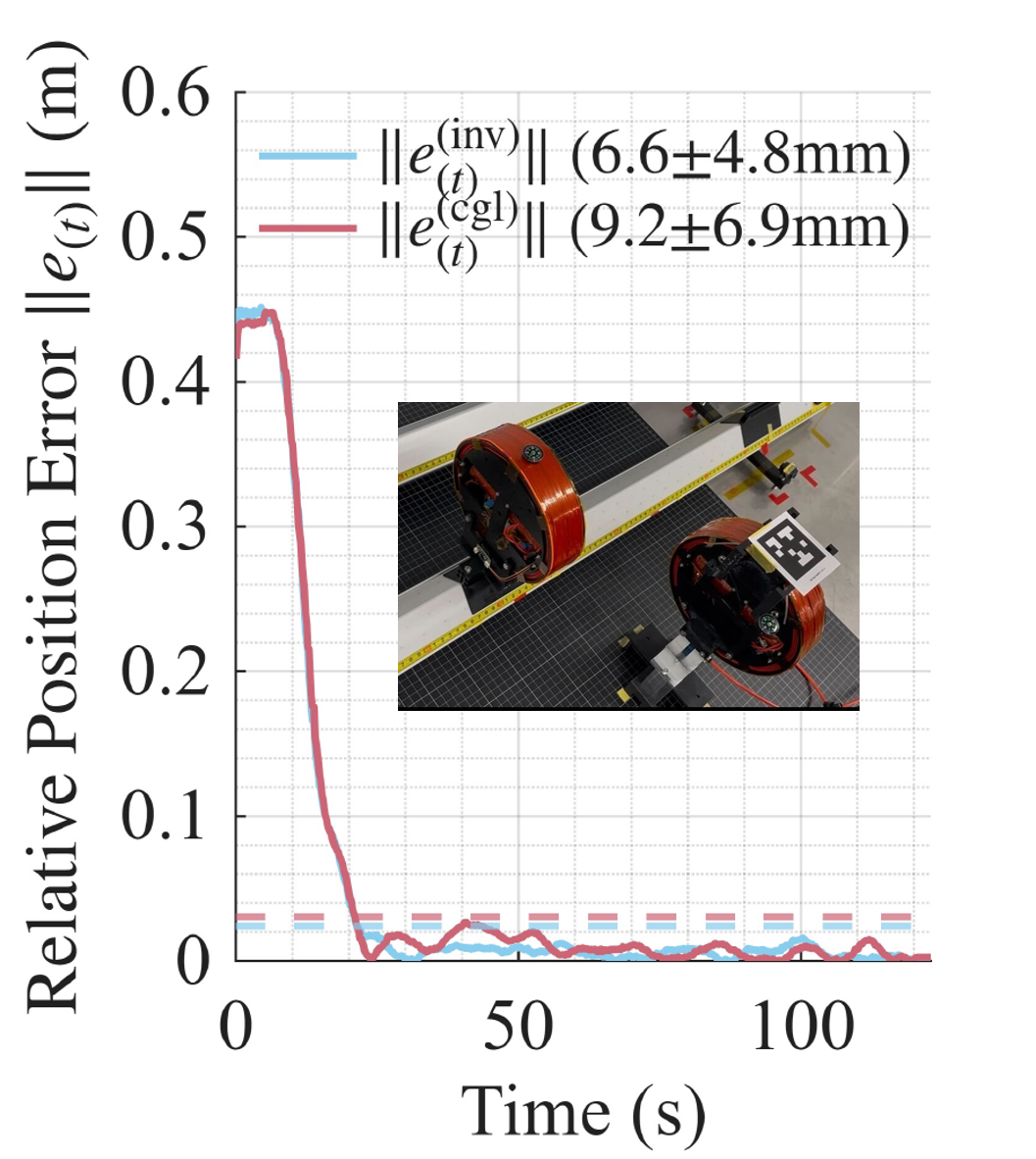}\label{fig:attitude}
\end{minipage}
\begin{minipage}[b]{0.49\columnwidth}
\centering
\includegraphics[width=1.0\linewidth]{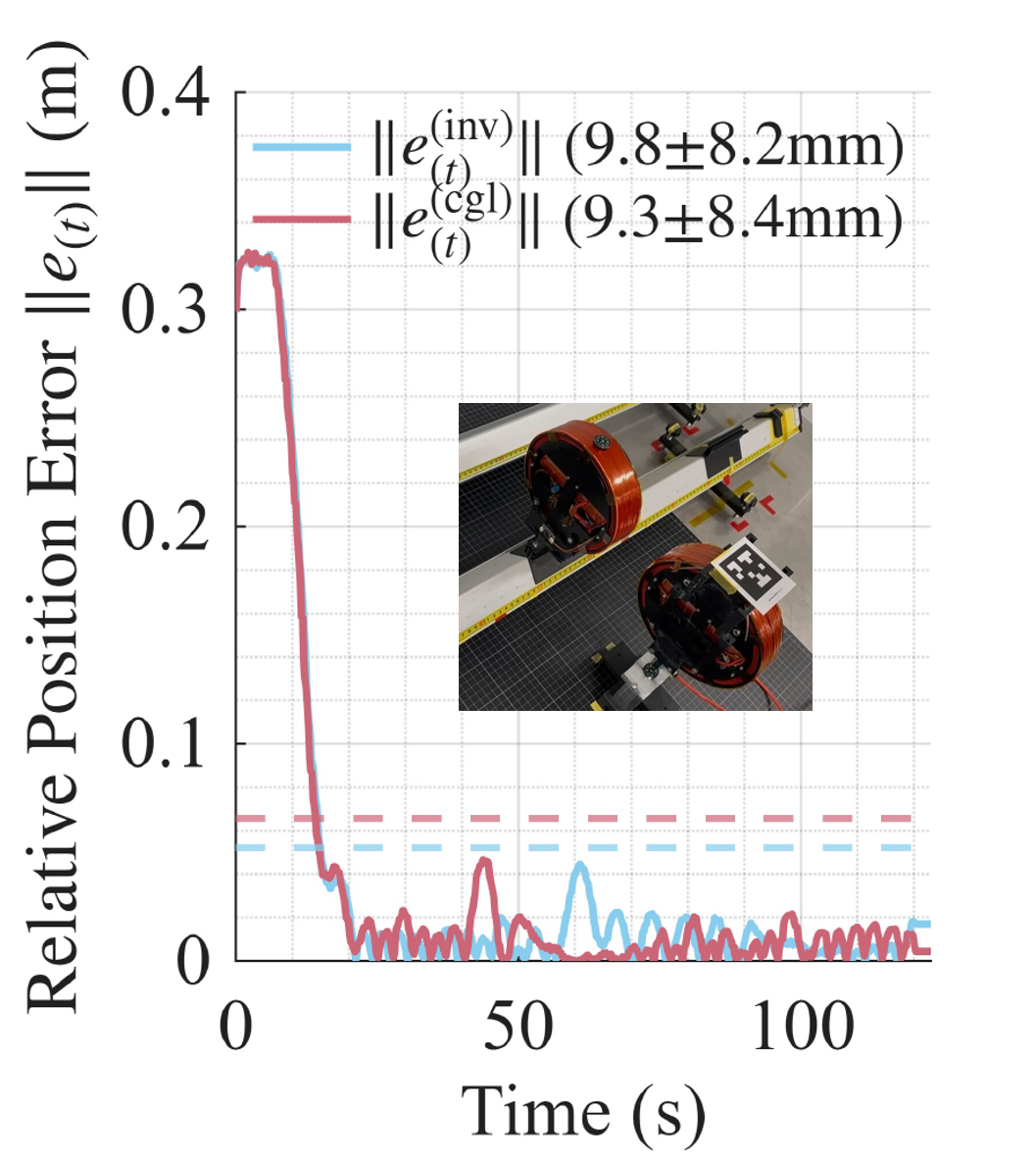}\label{fig:attitude}
\end{minipage}
\caption{Experimental results for two far target positions.}
\label{two_initial_condition_results_for_far_range_operation}
\end{figure}
\noindent
For the far-range operation in Fig.~\ref{two_initial_condition_results_for_far_range_operation}, the dipole approximation is relatively valid, and the learning model successfully demonstrates equivalent performance. At closer ranges, the model error of the dipole approximation produces a large position error bound control as shown in Fig.~\ref{fig_proximately_control_1_2}. The learning model shows a tighter error and bound. 
\begin{figure}[tb!]
\centering
\includegraphics[width=1.0\columnwidth]{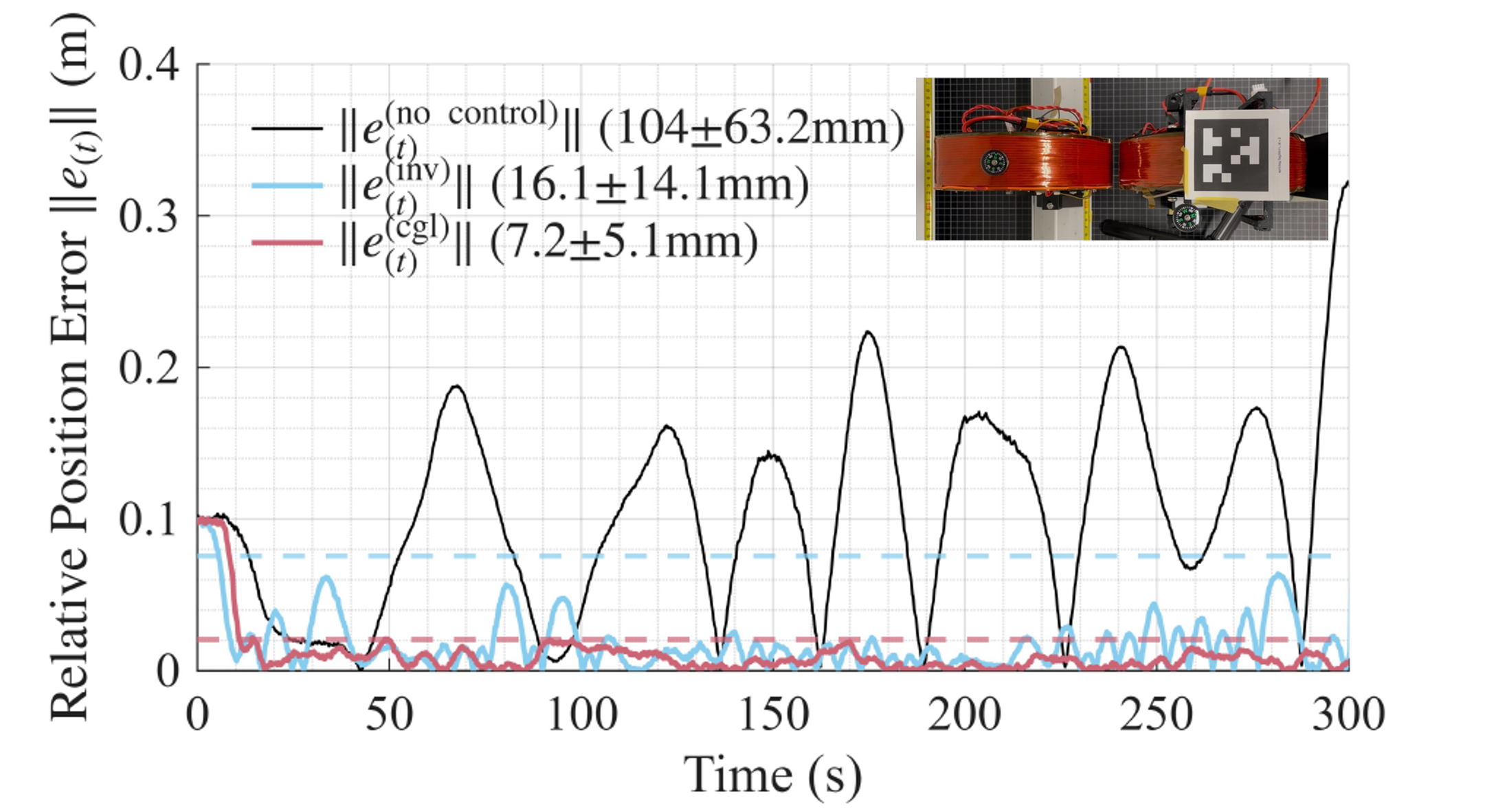}\label{fig_proximately_control_1}
\caption{Experimental results of proximity position control.}
\label{fig_proximately_control_1_2}
\end{figure}
\subsubsection{Simultaneous Control Trial: Two Satellite Docking}
We extend the experiment in Fig.~\ref{fig_proximately_control_1_2} for simultaneous control of position and attitude in Fig.~\ref{experiment_formation_and_angle_control_results}. The learning-based model successfully predicts the exact magnetic field for proximity operation and realizes the rapid convergence of the states. The dipole model-based control degrades the convergence rates due to its model error.
\begin{figure}[tb]
\centering
\begin{minipage}[b]{0.49\columnwidth}
\centering
\includegraphics[width=1\linewidth]{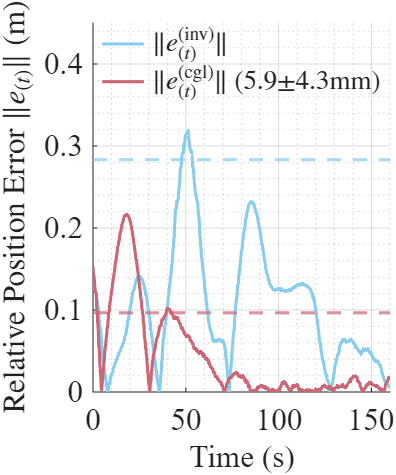}\label{fig:attitude}
\end{minipage}
\begin{minipage}[b]{0.49\columnwidth}
\centering
\includegraphics[width=1\linewidth]{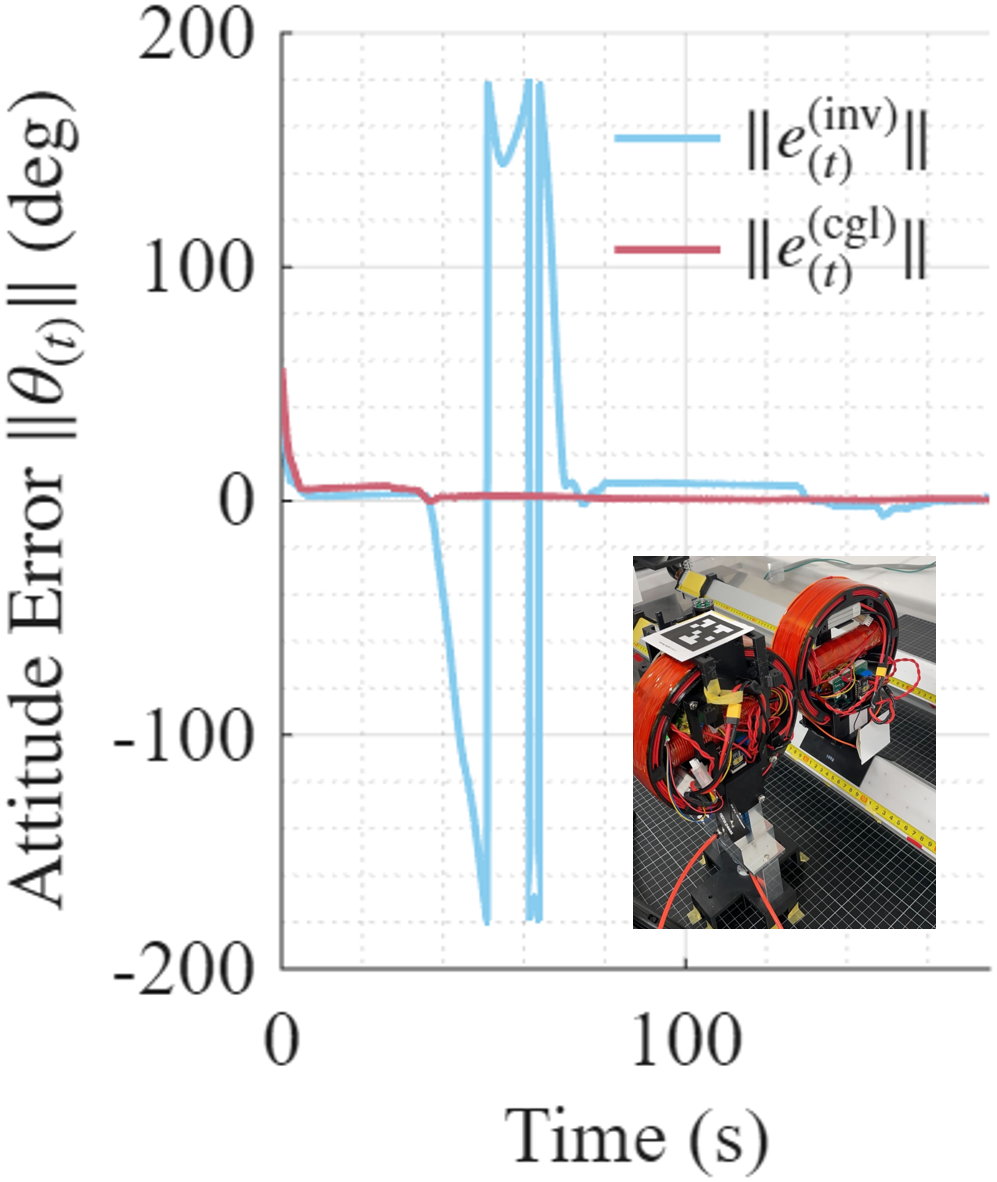}\label{fig:attitude}
\end{minipage}
\caption{Experimental results for position and attitude control.}
\label{experiment_formation_and_angle_control_results}
\end{figure}
\subsubsection{Scalability Trial: Three Satellite Formation Control}
We testify to the decentralized dipole allocation for the three satellites in Fig.~\ref{three_satellite_formation_control_results}. The pairs S/C1-2 and S/C2-3 are controlled by different frequencies in a decentralized manner. Their formation converges to the equilateral triangle formation.
\begin{figure}[tb]
\centering
\begin{minipage}[b]{0.49\columnwidth}
\centering
\includegraphics[width=1\linewidth]{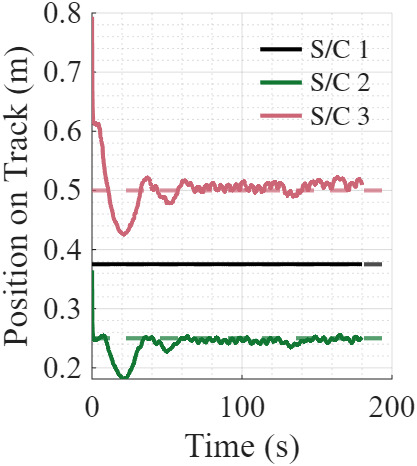}\label{fig:attitude}
\end{minipage}
\begin{minipage}[b]{0.49\columnwidth}
\centering
\includegraphics[width=1\linewidth]{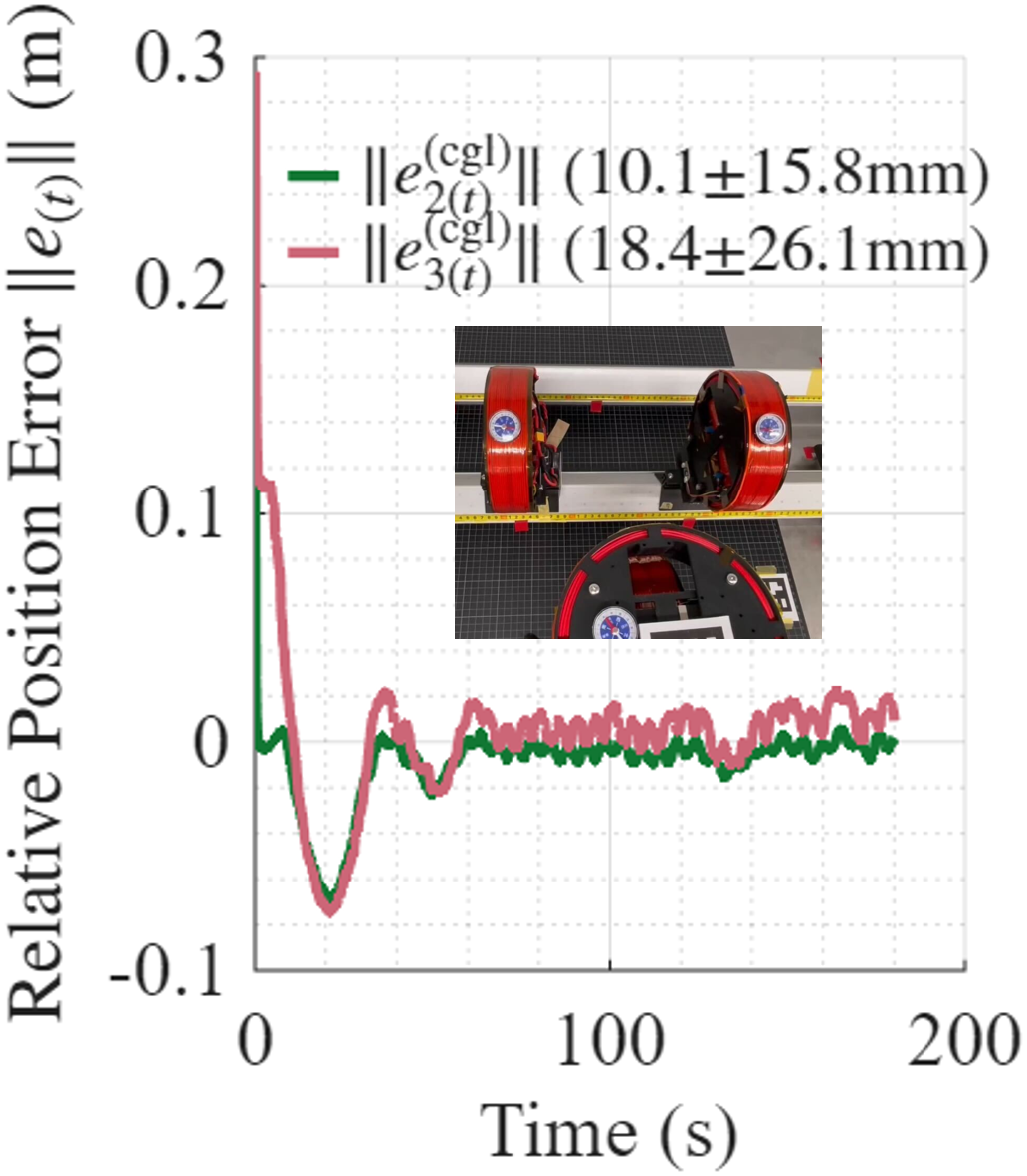}\label{fig:attitude}
\end{minipage}
\caption{Three coils control using decentralized dipole allocation.}
\label{three_satellite_formation_control_results}
\end{figure}
As an averaged $J_2$ orbital parameters error analysis, We choose $\Theta_{11}=\theta_{11}I$ and apply $\|E^\top D_y\|_\infty\approx  1e^{-8}$ \cite{komatsu2025real}, $\mathrm{d}t_{\mathrm{orb}}=10$ \cite{takahashi2025distance}, $k_0\approx 1.8e^{3}$ \cite{takahashi2025distance} to our parameters $\mathrm{d}t_{\mathrm{gnd}}=0.185$. These derive $\beta=55.6$ in Eq.~(\ref{coordination_matrix}) and, thus, $\theta_{11}
\approx {1e^{-3}}/{\| E^\top d\|_{\infty}}$ applied to the result in Fig.~\ref{fig:AOS}. Figure~\ref {fig:AOS} shows the connectable time $T_{\mathrm{conn}(t)}$ \cite{takahashi2025distance,takahashi2025scalable} with S/C 1 under no control execution. This evaluation is effective for the distance-based nature of magnetic field interaction.
\begin{figure}[tb]
\centering
  \centering
 \includegraphics[width=0.9\linewidth]{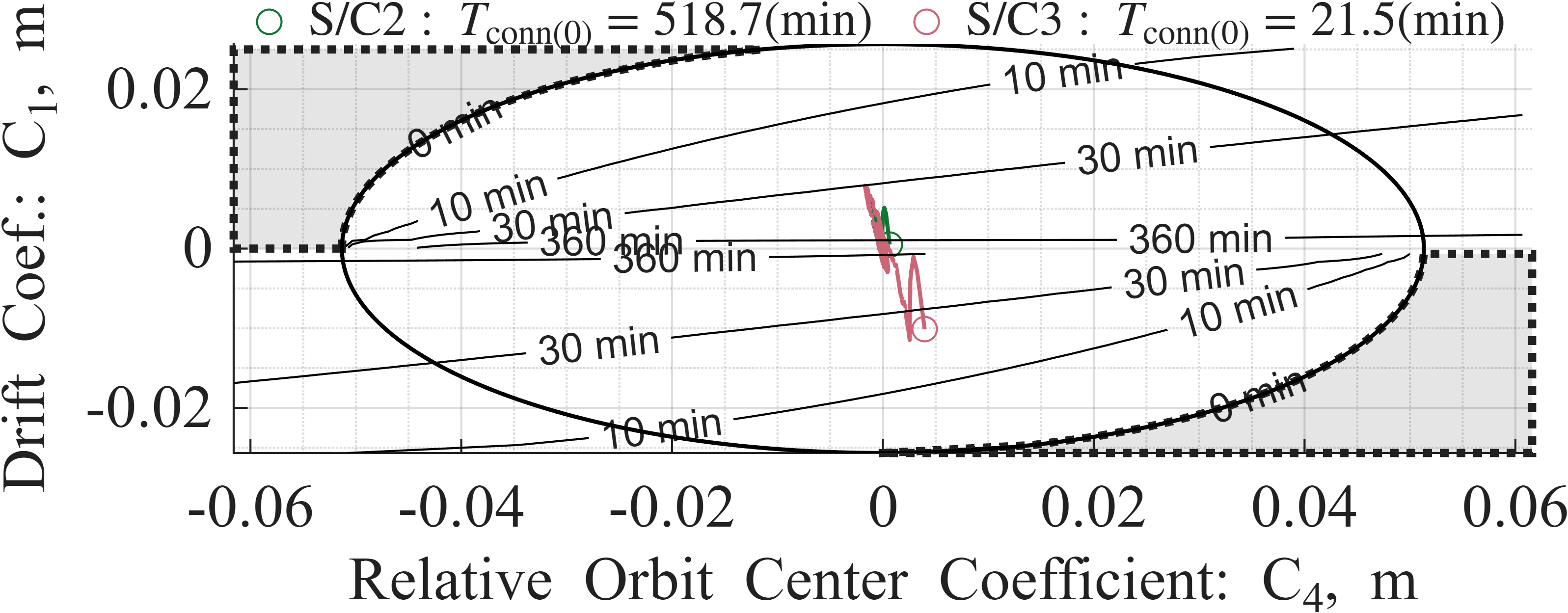}
  \caption{Connectable time analysis \cite{takahashi2025distance,takahashi2025scalable} at orbital dynamics.}
\label{fig:AOS}
\end{figure}
\section{Conclusion}
This study experimentally validates the principle of large-scale satellite swarm control through learning-aided magnetic field interactions generated by satellite-mounted magnetorquers. We introduce a power-optimal swarm control framework. Through the experimental validation, we validate critical aspects of time-varying magnetic actuation and learned magnetic model as the foundation of our framework. 
\appendix
\subsection{Proof: Theorem~\ref{theorem_NODA_MMH} (Input-to-State Stability Analysis)}
\label{proof_NODA_MMH}
\begin{proof}
Applying the dipole moment in Eq.~(\ref{designed_input}) into Eq.~(\ref{EMFF_dynamics_null_space}) yields the closed-loop system: $\overline{M}\dot{\mathsf{e}}_v+\overline{C}\mathsf{e}_v=-S^\top B_{\mathrm{e/f}}KS\mathsf{e}_v+S^\top d$.
For this nonlinear system, the partial contraction analysis \cite{wang2005partial} introduces the virtual dynamics of the parameterized path $y_{(\mu)}$ for $\mu\in[0,1]$: $\overline{M}_{\left(q_s\right)} \dot{y}_{(\mu)}+\overline{C}_{\left(q_s, v\right)}y_{(\mu)}=-S^\top B_{\mathrm{e/f}}KS y_{(\mu)}+\mu S^\top d$ where $y_{(0)}=0$ and $y_{(1)}=s$ are solutions for the unperturbed desired states and the described closed loop system states, respectively. We define the virtual displacement $\partial_\mu y={\partial y}/{\partial\mu}$ and its differential equation is $\overline{M}\partial_\mu \dot{y}+\overline{C} \partial_\mu y=-S^\top B_{\mathrm{e/f}}KS \partial_\mu y+S^\top d$. 
Then, for $V=\delta V_{l}^2$ and $\delta V_{l}=\|\sqrt{M}{S}{\partial_\mu y}\|$, the inequality in Eq.~(\ref{inequality_contraction}) $
(2\delta V_{l})\delta \dot{V}_{l}
\leq (2\delta V_{l})^\top(-\alpha\delta V_{l}+\sqrt{M}^{-1}d)
$ where ${\partial_\mu y^\top}(\dot{\overline{M}}-2\overline{C}){\partial_\mu y}=0$. This yields the differential inequality $\delta \dot{V}_{l}\leq -\alpha \delta V_{l} + \sqrt{M}^{-1}\sup_{t}\overline{\|d(t)\|}$ and all system trajectories in this virtual system converge exponentially fast to a single trajectory at a rate of $\alpha$. Applying the comparison lemma \cite{khalil2002nonlinear} into the path integral $V_{l(y,\delta y,t)}=\int_0^1\delta V_{l}\mathrm{d}\mu\geq \sqrt{\lambda_{\min}(M)}\|{S}\mathsf{e}_{v(t)}\|$ yields $\|{S}\mathsf{e}_{v(t)}\|\leq (V_{l(0)}e^{-\alpha t}+({\sup_{\zeta,t}\|d\|}/\alpha)(1-e^{-\alpha t}))/\sqrt{\lambda_{\min}(M)}$ and we get $\|{S}\mathsf{e}_{v(t)}\|\leq ({\sup_{\zeta,t}\|d\|}/\alpha)/\sqrt{\lambda_{\min}(M)}$. Based on the result in subsection~\ref{MRP_definition}, we obtain error dynamics of position and attitude: $\dot{\mathsf{e}}_q
=P_{(\mathsf{e}_{\sigma})}(Sv_r-\zeta_d)+P_{(\mathsf{e}_{\sigma})}S\mathsf{e}_{v(t)}
=-\Lambda \mathsf{e}_q+P_{(\mathsf{e}_{\sigma})}S\mathsf{e}_{v(t)}$. Applying the comparison lemma \cite{khalil2002nonlinear} to this equation yields the steady error in Eq.~(\ref{steady_error}).
\end{proof}
\subsection{Coil Design Flow and Experimental Requirements}
\label{Testbed_Design_and_Gravity_Compensation_Test}
\subsubsection{Sequential Magnetorquer Parameter Selection}
We have identified the following decision factors in coil design 
and others become dependent:
\begin{itemize}
    \item Inner and outer coil diameters $\underline{D},\overline{D}$ [m] for multilayer diameter $D_{\mathrm{coil}}=(\underline{D}+\overline{D})/2$ and height $H_{\mathrm{coil}}$ [m]
    \item Operational distance $r$ [m] with target and initial $r_d,r_0$
    \item Circuit voltage $V_{\mathrm{cir}}$ [V], Upper current $\overline{c}_{\mathrm{wire}}$ [A]
\end{itemize}
We select the appropriate wire to satisfy the design requirements based on the following material properties: 
${k}_{\Omega/kg} [\Omega/kg]$, $k_{\Omega/m} [\Omega/m]$, $\overline{c}_{\mathrm{wire}}[A]$, $D_{\mathrm{wire}} [m]$. We first calculate the coil resistor ${\Omega}_{\mathrm{coil}}$ to satisfy the current and mass constraint, and the dependent parameters follows: 
$$
\begin{aligned}
&\text{Coil resistor [$\Omega$]:\ }{\Omega}_{\mathrm{coil}}=\min (\overline{V}_{\mathrm{cir}}/\overline{c}_{\mathrm{wire}},\overline{m}_{\mathrm{coil}}{k}_{\Omega/kg})\\
&\left\{
\begin{aligned}
\text{The number of wiring [-]: }&N_t = {\Omega}_{\mathrm{coil}}/k_{\Omega/m}/(\pi D_{\mathrm{coil}})\\
\text{Maximum dipole [A/m$^2$]: }&\overline{\mu}= \pi (D_{\mathrm{coil}}/2)^2 N_t\overline{c}_{\mathrm{wire}}\\
\text{Thickness of coil wire [m]: }&t_{coil}=N_t/(H_{coil}/D_{wire})
\end{aligned}
\right.
\end{aligned}
$$
\subsubsection{Design Optimization for Experimental Constraint}
\begin{table}[tb!]
\centering
\caption{The 1U model and experimental MTQ designs.}
\begin{tabular}{c|c|c}
\hline
Parameter & Optimization \cite{shim2025feasibility}&1U model\\ \hline
Satellite/Coil size
&   11/7.5cmcm  & 10/7.5cm\\ \hline
Mass/Coil
& 0.54/0.10kg  & 0.34/0.12kg\\ \hline
\end{tabular} 
\label{tab:1U_MTQ_model}
\begin{tabular}{c|c|c|c|c|c}
\hline\hline
Diameters&$N_t$&$\overline{\mu}$&$I_z$&R&$\omega\ [rad/s]$\\\hline
14-16cm&120&12.5& 6.9$e^{-4}$&2&8$\pi$[1,2,3,4,5]\\\hline
\end{tabular} 
\label{tab:experimental_coil_design}
\end{table}
Our initial experiments are conducted under a far-field approximation to evaluate the feasibility of the proposed coil design with low computational costs. 
Then, the coil design requirements and specific values for our study are as follows: 
\begin{itemize}
    \item Size constraint 
    to hold the dipole model such as $r_d,r_0$ is greater than $2.5D_{\mathrm{coil}}$, e.g., ${r_d=3D_{\mathrm{coil}},\ r_0=r_d+0.1}$.
    \item Disturbance force $a_d$ [N/kg] by the linear track distortions $|a_d|\approx 2(9.8 \sin_{(0.01^{\circ})})\leq 4e^{-3}$. 
    \item Maximum weight $\overline{m}_{\mathrm{MTQ}}\approx 1.15$: that the linear air tracks can support 
    from measurements.
    \item 
    Discrete drive voltage $\overline{V}_{\mathrm{cir}}$ and mass of LiPo battery.
\end{itemize}
We solved the optimization to maximize acceleration. 
$$
\begin{aligned}
D_{\mathrm{coil}}^*,\\
\overline{V}_{\mathrm{cir}}^*
\end{aligned}
=\argmax_{\substack{D_{\mathrm{coil}},V_{\mathrm{cir}}\in\mathbb{R}}}\ \frac{\overline{\mu}^2}{m_{\mathrm{coil}}}{\mathrm{s.t. }\ }\left\{
\begin{aligned}
m_{\mathrm{coil}}&=\frac{\Omega_\mathrm{coil}}{k_{\Omega/kg}} \leq \overline{m}_{\mathrm{coil}}\\ 
F_{(d_0)}&=\mathbf{\frac{1}{2}}\frac{3\mu_0}{2\pi}\frac{\mu^2}{d_0^4} \geq a_d\\
t_{coil}&\leq \frac{D_{\mathrm{coil}}}{6}
\end{aligned}
\right.
$$
where $({V},\overline{c}_{\mathrm{wire}},{k}_{\Omega/kg},\overline{m}_{\mathrm{coil}})$. 
We consider three different wire materials: copper, polyester, and tin-plated copper. For the high current requirement, optimal calculation selects 
a UL1007 AWG20 wire, whose maximum permissible current is 7 A, and an 11V battery supply, as summarized in Table~\ref{tab:experimental_coil_design}. The second coil utilizes an iron core to address size constraints and generate an equivalent magnetic field.
\begin{remark}
Our comprehensive system design for a satellite swarm based on nonconvex optimization \cite{shim2025feasibility} verified that the 1U MTQ shown in Table \ref{tab:1U_MTQ_model} and Fig.~4 meets the long-term formation-keeping requirement for low-Earth orbit communication.
We utilize the larger model to address the effects of microgravity on the linear air track, along with the evaluation criteria introduced in subsection~\ref{Evaluation_Criteria_Toward_Orbit_Proof}.
\end{remark}
\begin{figure}[bt]
    \centering
\includegraphics[width=0.9\linewidth]{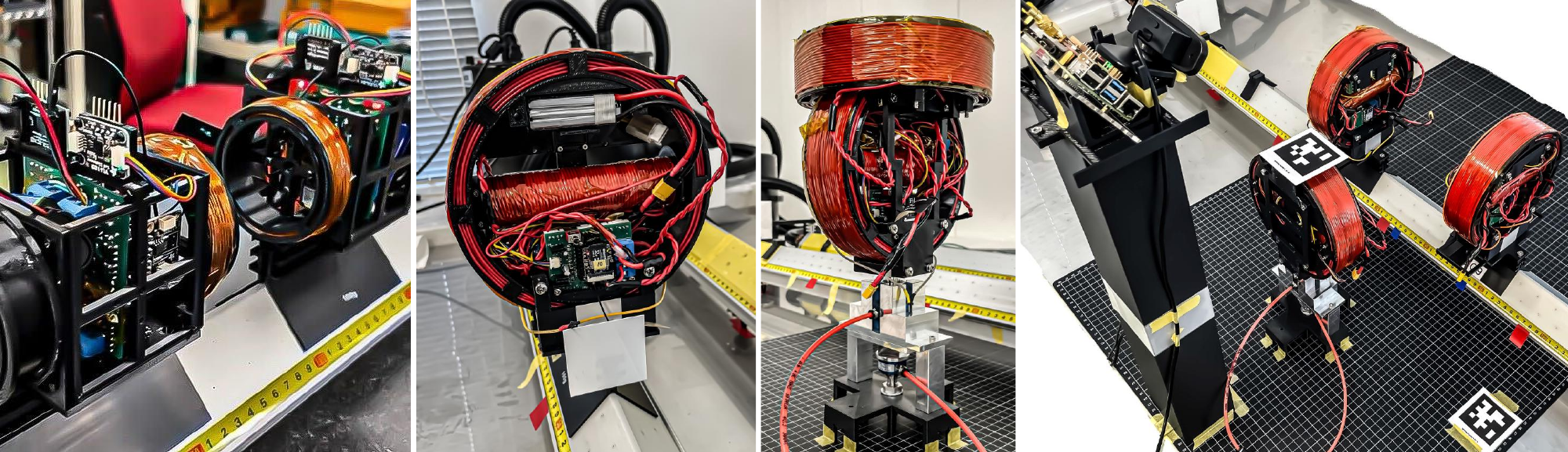}      
  \caption{1U and experimental model on air track and bearing.}
  \label{experiment_picture}
\end{figure}
\subsubsection{Time Constant Requirement for Orbital Operation}
\label{relative_orbital_dynamics}
We describe the relative orbital dynamics of the orbiting satellites using an orbitally fixed system $\mathcal{O}$:
${\mathbf{o}_x}=\mathrm{nor}(\mathbf{r})$, ${\mathbf{o}_y}={\mathbf{o}_z} \times {\mathbf{o}_x}$, and ${\mathbf{o}_z}=(\mathbf{r} \times \dot{\mathbf{r}})/\|\mathbf{r} \times \dot{\mathbf{r}}\|$ where $\mathbf{r}\in\mathbb{R}^3$ represents the position vector from the center of the Earth. We define the relative position of the $j$-th satellite from the $k$th satellite as ${r_{jk}}={r}_j-{r}_k$. 
Analitical solution of the linearized $J_2$ perturbed relative motion dynamics Integrating this yields \cite{takahashi2025distance,takahashi2025scalable}:
$r_{jk(t)}=[{2C_{1}}+r_{xy}\sin{(\omega_{xy} t + \theta_{xy})}/c_+;{C_{4}-\epsilon_{2} C_{1} t}+2r_{xy}\cos{(\omega_{xy} t + \theta_{xy})}/c_-;(r_z+l_z t) \sin{({\omega}_z t+\theta_z)}]$ where 
$c_{\pm}\approx\sqrt{1\pm 1e^{-4}}$,  
$\epsilon_2\approx(3+5e^{-4})\omega_{xy}$, ${\omega}_{xy}=c_-\sqrt{{\mu}/{r_{\mathrm{ref}}^3}}$, $C_{1,4}$ is the averaged $J_2$ relative orbital parameters about the drift motion and relative orbit center \cite{takahashi2025distance,takahashi2025scalable}:
$$
C_{1}=({c_+}/{c_-^2})\left(2 \overline{x}+{\dot{\overline{y}}}/{ {\omega}_{xy}}\right),\ C_{4}=({1}/{c_-})\left(\overline{y}-{2 \dot{ \overline{x}}}/{{\omega}_{xy}}\right).
$$
This drift term $C_1$ contributes to MTQ control degradation. 
\subsection{Proof: Theorem~\ref{theorem_Dynamics_Normalization} (Unstable Dynamics Normalization)}
\label{proof_Dynamics_Normalization}
\begin{proof}
The constraint in Eq. ~(\ref{coordination_matrix}) is equivalent to 
$$
\begin{aligned}
&\text{(i)} -\Theta_{11}E^\top K_d+\Theta_{12}= \beta A_{11}\Theta_{11},\ \text{(ii)} \Theta_{22}= \beta A_{21}\Theta_{11}\\
&\text{(iii)} -\Theta_{11}E^\top K_p = \beta A_{11}\Theta_{12},\ \text{(iv)} 0      = A_{21}\Theta_{12}+A_{22}\Theta_{22}\bigr.
\end{aligned}
$$
Since (iii) and (iv) limit that $\Theta_{11}$ commute with $L_e$ such as an invertible polynomial function $\Theta_{11}=P(L_e)$ and this determines $\Theta_{12,22}$ in Eq.~(\ref{Theta_11}). Substituting $\Theta_{12}$ into (iii) and (iv) and canceling the commuting, the nonsingular $\Theta_{11}$ condition results $\beta^{2}A_{11}^{2} + \beta A_{11}E^\top K_d + E^\top K_p=0$ and $E^\top K_d=\beta \frac{k_A(1+\gamma )}{2}L_e$, thus, $E^\top K_p=\frac{\beta^2 k_A^2\gamma}{4}L_e^2$.
\end{proof}
\subsection{Proof: Corollary~\ref{Corollary_Error_Bound_Control} (Steady Position Error Bound)}
\label{Proof_Error_Bound_Control}
Toward orbital implementation, we derive the steady position error bound under space radiation. 
\begin{proposition}
Consider the MLP $\mathcal{F}_0$ processed $P$-th order residual quantization for $n_{\mathrm{bit}}$-bit with $p$-th order protection, i.e., $W^l\approx\sum_{i=1}^{P}W^{l(i)}$. If $n_{\mathrm{bf}}$ bit flips occur during $T$, $\|\mathcal{F}_T\|_{\text {Lip }}$ is bounded associated $\overline{\mathcal{F}_{0\text {Lip}}}$ in subsection~\ref{MLP_introduction}
\begin{equation}
\label{Lipschitz_ratio}
\gamma_{\mathcal{F}_T}=\frac{\|\mathcal{F}_T\|_{\text {Lip }}}{\overline{\mathcal{F}_{0\text{Lip}}}}
\leq\ \prod_{l=1}^{L+1}\left(1+\frac{ n_{\mathrm{bf}}\ \max_{i,j} |w_{0ij}^{l}|}{2^{(p n_{\mathrm{bit}}-1)}\ \sigma(W_0^l)}\right)
\end{equation}
\end{proposition}
\begin{proof}
The Lipschitz constant of the function $\mathcal{F}$ at time $T$ is bounded as $\|\mathcal{F}_T\|_{\text {Lip }} 
\leq \overline{\mathcal{F}_{\text {Lip}}}\prod_{l=1}^{L+1}(1+{\|\Delta W_T^l\|_F}/{\sigma(W_0^l)})
$. 
Since we assume $p$-th order protection, we have $\|\Delta W_T^l\|_F\leq  n_{\mathrm{bf}}(2\max_{i,j} |w_{0ij}^{l(p+1)}|)$ due to sign flip. Residual quantization error for any matrix $A$ is up to $\max_{i,j} |a_{ij}|/(2\times 2^{(n_{\mathrm{bit}}-1)})$ 
and, thus, $\max_{i,j} |w_{0ij}^{l(p+1)}|\leq \max_{i,j} |w_{0ij}^{l}|/(2\times 2^{(n_{\mathrm{bit}}-1)})^p$. Then, these leads to  
$\gamma_{\mathcal{F}_T}$ in Eq.~(\ref{Lipschitz_ratio}).
\end{proof}
\begin{corollary}[Steady Position Error Bound]
\label{Corollary_Error_Bound_Control}For a given constant gain $K=K_{\mathrm{pos}}\oplus K_{\mathrm{rot}}$, the steady position error is $\lim_{t\rightarrow\infty}\|{\mathsf{e}}_{p(t)}\|\leq {\sup_{t}\|d(t)\|}/{(\alpha\sqrt{{\mathrm{m}}\ {k_p}/{k_d}})}$ where  $\sup_{t}\|d(t)\|=\sup_{t\in t_{\mathrm{steady}}}\sup_{\tau\in[t,t+T)}\|{d}_{j(\tau)}^{2\omega}\|+(L_{G_\zeta} + (1+\gamma_{\mathcal{F}_T})L_{\mathfrak{G}_\zeta})\rho_{r,\sigma}$, $\mathrm{m}$ is the mass, $\alpha ={k_d}/{\mathrm{m}}$, $\rho_{r,\sigma}$ is the covering radius about position and attitude, $\gamma_{\mathcal{F}_T}$ is the degradation ratio in Eq.~(\ref{Lipschitz_ratio}), and Lipschitz constrants $L_{G_\zeta,\mathfrak{G}}$ about true and learned function of inputs using $G_\zeta$.
\end{corollary}
\begin{proof}
We obtain the bound by Eqs.~(\ref{steady_error}) and~(\ref{result}). For a sample region $\mathcal{S}_{R}$ and training set $\mathcal{T}_{R}$, the analytical bound of learning error $\Delta_{\mathrm{NN}(x)}$ after model quantization is $\sup_{x\in \mathcal{S}_{R}}\|\Delta_{\mathrm{NN}(x)}\|\leq\sup_{x'\in\mathcal{T}_{R}}\| \Delta_{(x')}\|+ (L_{G_\zeta} + (1+\gamma_{\mathcal{F}_T})L_{\mathfrak{G}_\zeta})\rho_{r,\sigma}$.
\end{proof}
\section*{Acknowledgments}
This work was partially supported by the JAXA, Space Strategy Fund GRANT Number (JPJXSSF24MS09003), Japan. The first author would like to thank Hannya Yamato, Masaru Ishida, Yusuke Sawanishi, and Takahiro Inagawa at Interstellar Technologies Inc. for their technical support.
\bibliographystyle{IEEEtran}
\bibliography{references_1023_2025}
\end{document}